\documentclass[10pt,twocolumn,letterpaper]{article}

\usepackage[pagenumbers]{wacv} 

\usepackage{graphicx,enumitem}
\usepackage{booktabs}
\usepackage{float}
\usepackage{makecell}
\usepackage{soul}
\usepackage{graphicx}
\usepackage{multirow}
\usepackage{xcolor}
\usepackage{colortbl}

\usepackage{amsmath,amssymb,amsfonts, mathtools,mathrsfs, amsthm}
\newtheorem{theorem}{Theorem}
\newtheorem{lemma}{Lemma}
\newtheorem{definition}{Definition}
\newtheorem{assumption}{Assumption}

\theoremstyle{remark}
\newtheorem*{remark}{Remark}

\usepackage{algorithm}
\usepackage{algcompatible}
\makeatletter
\def\ALG@special@indent{%
    \ifdim\ALG@thistlm=0pt\relax
        \hskip-\leftmargin
    \else
        \hskip\ALG@thistlm
    \fi
}

\newcommand{\ClientUpdate}[1]{\item[]\noindent \ALG@special@indent \textbf{ClientUpdate($ \theta^t$):}\ #1}
\usepackage{todonotes}
\newcounter{todocounter}

\usepackage[pagebackref,breaklinks,colorlinks]{hyperref}

\usepackage[capitalize]{cleveref}
\crefname{section}{Sec.}{Secs.}
\Crefname{section}{Section}{Sections}
\Crefname{table}{Table}{Tables}
\crefname{table}{Tab.}{Tabs.}

\def\wacvPaperID{1461} 
\def\confName{WACV}
\def\confYear{2025}

\begin{document}

\title{Achieving Byzantine-Resilient Federated Learning via Layer-Adaptive Sparsified Model Aggregation}

\author{Jiahao Xu\\
University of Nevada, Reno\\
{\tt\small jiahaox@unr.edu}
\and
Zikai Zhang\\
University of Nevada, Reno\\
{\tt\small zikaiz@unr.edu}
\and
Rui Hu\\
University of Nevada, Reno\\
{\tt\small ruihu@unr.edu}
}
\maketitle

\begin{abstract}
Federated Learning (FL) enables multiple clients to collaboratively train a model without sharing their local data. Yet the FL system is vulnerable to well-designed Byzantine attacks, which aim to disrupt the model training process by uploading malicious model updates. Existing robust aggregation rule-based defense methods overlook the diversity of magnitude and direction across different layers of the model updates, resulting in limited robustness performance, particularly in non-IID settings. To address these challenges, we propose the Layer-Adaptive Sparsified Model Aggregation (LASA) approach, which combines pre-aggregation sparsification with layer-wise adaptive aggregation to improve robustness. Specifically, LASA includes a pre-aggregation sparsification module that sparsifies updates from each client before aggregation, reducing the impact of malicious parameters and minimizing the interference from less important parameters for the subsequent filtering process. Based on sparsified updates, a layer-wise adaptive filter then adaptively selects benign layers using both magnitude and direction metrics across all clients for aggregation. We provide the detailed theoretical robustness analysis of LASA and the resilience analysis for the FL integrated with LASA. Extensive experiments are conducted on various IID and non-IID datasets. The numerical results demonstrate the effectiveness of LASA. Code is available at \url{https://github.com/JiiahaoXU/LASA}.

\end{abstract}


\section{Introduction}
Federated Learning (FL)~\cite{FL_OG_shake} is an emerging distributed machine learning paradigm that enables multiple clients, such as mobile devices or organizations, to collaboratively train a shared model while keeping their private data locally. This approach significantly reduces the necessity for data centralization, thereby not only decreasing data communication costs but also mitigating data privacy concerns. 
FL framework has been applied in diverse fields such as healthcare~\cite{nguyen2022federated} and remote sensing~\cite{liu2020federated}, facilitating the use of machine learning in scenarios where data privacy and communication efficiency are critical. 

However, distributing the model training across individual clients makes FL vulnerable to poisoning attacks~\cite{bulyan, liu2022evil, tian2022comprehensive}, where an attacker controls a subset of clients and manipulates their local model updates to compromise the integrity of the global model. The \textit{Byzantine attack}~\cite{OG_byz,bulyan, fang2020local, dnc, lie, signguard} is one of the potent attacks which generally degrades the model's overall performance during the training. Specifically, under Byzantine attacks, a small set of malicious clients sends corrupted local model updates to the server during the training. It is shown that if the aggregation rule used by the server is a simple linear combination of local model updates, even one single malicious client can easily destroy the convergence of the global model~\cite{krum_mkrum_bulyan}. Therefore, many efforts have been dedicated to designing aggregation rules that are robust against Byzantine attacks. 

Existing robust aggregation rules can be mainly categorized into two types based on their granularity in handling model parameters: coordinate-wise robust aggregation~\cite{krum_mkrum_bulyan, trmean, trmean2} and model-wise robust aggregation~\cite{geomed, krum_mkrum_bulyan, bulyan, dnc, signguard}. Coordinate-wise robust aggregation focuses on evaluating and aggregating each coordinate of model updates independently, effectively filtering out extreme values that could represent malicious activity with fine granularity. 
In contrast, model-wise robust aggregation holistically assesses the entire model update from each client to detect outliers. 
The primary challenge in designing an effective robust aggregation rule lies in the difficulty of distinguishing between benign and malicious model updates, especially when the attacker's manipulations are subtle enough to blend seamlessly with benign data. This challenge becomes more pronounced in FL settings with non-IID data. To capture the subtle difference between benign and malicious model updates in such cases, it is essential to strike a balance between fine-grained and holistic assessment of model updates. 

Recently, model sparsification has been used as an approach to enhance the Byzantine robustness of FL~\cite{panda2022sparsefed, VRDP, meng2023enhancing}. The key idea is to remove less important parameters in the model updates to alleviate the malicious impact while maintaining the model's utility. 
For instance, SparseFed~\cite{panda2022sparsefed} removes the parameters with less importance on the aggregated model update to defend against Byzantine attacks. 
However, current methods typically utilize a uniform sparsification mask for all model updates, leading to high sparsification error and limited robustness improvement, especially in non-IID settings where model updates diverge and necessitate personalized sparsification to preserve model utility.

Inspired by the sparsification-based method and robust aggregation with different granularity, we propose a novel Byzantine robust aggregation rule called \textbf{LASA} (\underline{\textbf{L}}ayer-\underline{\textbf{A}}daptive \underline{\textbf{S}}parsified Model \underline{\textbf{A}}ggregation) that achieves Byzantine robustness of FL at the granularity of layer-level and important parameters only. 
Basically, LASA at first sparsifies each local model update individually without degrading the model utility. These sparsified updates are then fed into a layer-wise filter to adaptively detect and drop potential malicious layers. Finally, the remaining layers will be averaged as the global model update. The model sparsification is applied before aggregation to reduce the attack surface of malicious clients and maintain the model utility of benign clients with personalized sparsification. It also enables the subsequent filtering to focus on key parameters that determine the model performance. Notably, this strategy is particularly beneficial for non-IID settings, where local model updates of benign clients are diverse. The layer-wise filter extracts both the magnitude and direction of a sparsified layer as metrics and also enables layer-adaptive filtering with minimal control parameters, allowing LASA to strike a balance between coordinate-wise filtering and model-wise filtering efficiently and achieve better robustness. 
The model sparsification is carefully co-designed with the layer-wise adaptive filtering to ensure its amplification effect on robustness. The main contributions of this work are summarized as follows:
\begin{itemize}[leftmargin=*]
    \item We propose a novel robust aggregation rule called LASA. \textit{To the best of our knowledge, our work is the \textbf{first to combine pre-aggregation model sparsification with layer-wise adaptive aggregation} to defend against Byzantine attacks in FL.} LASA can be easily integrated into the existing FL frameworks.
    \item We introduce a robustness criterion named $\kappa$-robustness, which quantifies the ability of an aggregation rule to accurately estimate the average of honest clients' inputs when $f$ out of $n$ clients are malicious. We prove that LASA is a $\kappa$-robust aggregation rule with $\kappa = O(c_k (1+f/(n-2f))$, where $c_k\leq1$ correlates with the sparsification level, demonstrating the effectiveness of sparsification in amplifying robustness. 
    Based on the robustness analysis, we also provide the resilience analysis of the local SGD-based FL algorithm with LASA, for general non-convex loss functions and in the context of non-IID data. \textit{To the best of our knowledge, our work is the \textbf{first to theoretically analyze layer-wise defense methods.}}
    \item We empirically evaluate the performances of LASA by conducting comprehensive experiments on both IID and non-IID datasets under various SOTA attacks. Compared to the SOTA defense methods, LASA achieves better robustness as well as performance.
\end{itemize}

\section{Related Works}\label{sec:bg_moti}

\textbf{Poisoning attacks to FL.} Federated Averaging (FedAvg)~\cite{FL_OG_shake} stands as the classic FL method in non-adversarial environments. Yet, it has a critical vulnerability: the global model within FedAvg is susceptible to arbitrary manipulation by even a single malicious client~\cite{krum_mkrum_bulyan, trmean}. In particular, such a client can mislead the convergence of the global model by poisoning its local update sent to the server, which is known as poisoning attack in the context of FL~\cite{krum_mkrum_bulyan, fang2020local, lie, bulyan, mahloujifar2019universal, xie2020fall, dnc, signguard, bagdasaryan2020backdoor, bhagoji2019analyzing, wang2020attack, xie2020dba, munoz2017towards, jagielski2018manipulating, tolpegin2020data, gupta2023novel, sun2021data, fereidooni2023freqfed}. 

Poisoning attacks can be categorized into \textit{untargeted attacks} and \textit{target attacks}. Targeted attacks (aka \textit{backdoor attacks}) aim to mislead the global model to incorrectly predict certain outcomes chosen by the attacker for specific inputs while keeping the model's performance on other inputs unaffected~\cite{bagdasaryan2020backdoor, bhagoji2019analyzing, tolpegin2020data, wang2020attack, xie2020dba}. Untargeted attacks (aka \textit{Byzantine attacks}) aim to generally disrupt the overall performance of the global model without any specific focus~\cite{fang2020local, lie, bulyan, mahloujifar2019universal, xie2020fall, dnc, signguard}. In this work, we focus on the Byzantine attacks on FL. The technical details of the SOTA Byzantine attacks~\cite{lie, dnc, signguard} are given in Appendix~\ref{attackmethods}. 

\textbf{Existing defense methods.}\label{relatedwork}
 Existing defense methods against Byzantine attacks in FL can be generally categorized into three types: 1) auxiliary data-based methods~\cite{park2021sageflow, cao2020fltrust, xie2019zeno} which leverage the proxy dataset to conduct additional evaluation and thus filter out updates with abnormal performance. However, these methods somehow contradict the privacy-preserving goal of FL as they require a server dataset that is similar to local data to help identify malicious updates. Note that our approach does not need an auxiliary dataset and is orthogonal to these methods. 2) sparsification-based methods~\cite{panda2022sparsefed, VRDP, meng2023enhancing} which aim to remove malicious model parameters to enhance the robustness. For example, \textit{SparseFed}~\cite{panda2022sparsefed} sparsifies the aggregated model update at the server side, integrating with model clipping and error feedback, to mitigate the impact of malicious local model updates. Model sparsification can enhance robustness by reducing malicious parameters, but since the server can't identify malicious clients, it sparsifies all model updates, which degrades benign models' performance. Moreover, existing works~\cite{panda2022sparsefed, VRDP} use uniform sparsification masks which increase sparsification errors, especially in non-IID settings. Our method applies individual sparsification to each update and combines it with magnitude and direction-based filtering to boost robustness. In addition, we theoretically analyze how this model sparsification contributes to robust aggregation, bridging the gap in the current SOTAs. 3) Robust aggregation-based methods~\cite{trmean, trmean2, geomed, krum_mkrum_bulyan, bulyan, AFA,fang2020local, dnc, signguard, xu2022byzantine, varma2021legato} focus on developing a new aggregation rule on the server side that is robust against Byzantine attacks to replace the standard \textit{averaging} aggregation rule used in FedAvg. For example, 
\textit{Trimmed mean} (TrMean) proposed in \cite{trmean} discards a certain percentage of the highest and lowest values among the received models for each dimension. After this trimming, the mean of the remaining values is computed by the server, which mitigates the impact of extreme values on the aggregated model. 
\textit{Multi-Krum}~\cite{krum_mkrum_bulyan} selects the most reliable local model that has the smallest sum of squared Euclidean distances to all other models as the output. 
\textit{LEGATO}~\cite{varma2021legato} weights each layer before aggregation but cannot eliminate malicious parameters.
Recently, a defense method called \textit{SignGuard} that achieves {SOTA} results has been proposed in \cite{signguard}. It combines direction-based clustering and magnitude-based filtering to identify malicious model updates.

However, coordinate-wise methods~\cite{krum_mkrum_bulyan, trmean, trmean2} ignore model direction, and model-wise methods~\cite{geomed, krum_mkrum_bulyan, bulyan, dnc, signguard} overlook the diverse distribution of direction and magnitude across layers, limiting their robustness. Our layer-level approach is finer-grained than model-wise methods and more comprehensive than coordinate-wise methods. Furthermore, most works assume IID data, using clustering or distance-based methods to filter outliers~\cite{signguard, krum_mkrum_bulyan, xu2022byzantine, geomed}. However, in real-world FL scenarios, data is often non-IID, making these methods less effective. Our approach combines pre-aggregation model sparsification with layer-wise direction- and magnitude-based filtering to handle diverse model updates with only key parameters. We use a novel sign-based metric to assess model update directions, improving the purity and effectiveness of direction-based filtering. Unlike works like SignGuard~\cite{signguard}, we provide a detailed theoretical robustness analysis of LASA and its resilience in FL. Our theoretical analysis is most related to works on distributed gradient descent (D-GD)~\cite{kappa, allouah2023robust, allouah2024robust}, but we focus on FL with local SGD, which increases local model divergence and complicates the analysis.

\section{Our Solution: LASA}

The LASA process is given in Algorithm~\ref{alg:main}. 
LASA features an innovative design and integration of pre-aggregation model sparsification and layer-wise robust aggregation on the server side, aimed at mitigating the impact of malicious local model updates. 

\begin{algorithm}[t]
    \caption{LASA}
    \label{alg:main}
    \begin{algorithmic}[1] 
    \REQUIRE Set of $n$ local model updates $\{\Delta_i\}_{i=1}^{n}$, number of model layers $L$, sparsification parameter $k$, magnitude-based radius $\lambda_{m}$, and direction-based radius $\lambda_{d}$
        \FOR{$i\in[n]$} 
            \STATE $\hat{\Delta}_i  \leftarrow\text{\textit{Top}}_k(\Delta_i)$ $\hfill\lhd$ model sparsification \label{lasa_topk}
        \ENDFOR

        \FOR{each layer $l\in[L]$}  
            \STATE Initialize benign set $\mathcal{S}=\emptyset$
            \STATE $ \omega^l \leftarrow \{\textit{$L_2$-norm}(\hat{\Delta}_i^l)\}_{i=1}^n $ \label{lasa_l2}
            \STATE \label{lasa_pdp} $ \rho^l \leftarrow \{\textit{PDP}(\hat{\Delta}_i^l )\}_{i=1}^n \hfill\lhd$ by Equation.~\ref{eq:PDP} 
            \FOR{$i \in [n]$}
                \STATE \label{lasa_mz_m} $\lambda_{i,m}^{l} \leftarrow \textit{MZ-score}(\omega_i^l, \omega^l) \hfill\lhd$ by Equation.~\ref{eq:balancescore} 
                \STATE $\lambda_{i,d}^l \leftarrow \textit{MZ-score}(\rho_i^l,\rho^l)\hfill\lhd$ by Equation.~\ref{eq:balancescore}
                \IF{$|\lambda_{i,m}^{l}| \leq \lambda_{m}$ and $|\lambda_{i,d}^l| \leq \lambda_d$} \label{lasa_mz_d}
                    \STATE $ \mathcal{S} \leftarrow \mathcal{S} \cup \{i\}$ \label{lasa_12}
                \ENDIF
            \ENDFOR
            \STATE $ \bar{\Delta}^l\leftarrow \frac{1}{|\mathcal{S}|}\sum_{i\in \mathcal{S}}{\hat{\Delta}}_i^l \hfill\lhd$ layer-wise aggregation \label{lasa_layerwiseagg}
        \ENDFOR
        \STATE \textbf{return} $\bar{\Delta}$
        
    \end{algorithmic}
    
\end{algorithm}

\textbf{Pre-aggregation sparsification.} Specifically, in each round of FL, after receiving the local model updates from clients, the server first sparsifies each local model update $\Delta_i$ individually using the Top-$k$ sparsifier defined in Definition~\ref{def:spar} (line~\ref{lasa_topk}). 
\begin{definition}[Top-$k$ sparsifier~\cite{ hu2023federated}]\label{def:spar}
    For a vector $x \in \mathbb{R}^d$ and a parameter $k \in[1,d]$ , the Top-$k$ sparsifier $\text{Top}_k ( \cdot )$: $\mathbb{R}^d \rightarrow \mathbb{R}^d$ is defined as: $[\text{Top}_k(x)]_j \coloneqq [x]_{\pi(j)}$ if $ j \leq k $ and $[\text{Top}_k(x)]_j \coloneqq 0$ otherwise, 
    where $\pi$ is a permutation of indices such that $|[x]_{\pi(j)}| \geq |[x]_{\pi(j+1)}|$, $\forall j\in [1, d-1]$.
\end{definition}
With sparsification, each element of the sparsified model update $[\hat{\Delta}_i]_j$ equals to $ [\Delta_i]_j$ if $ j \in \mathcal{K}$ and 0 otherwise, where $\mathcal{K}$ represents the set of coordinates of parameters that have the top $k$ largest absolute values. Here, model sparsification can limit the attack surface available to malicious clients by dropping out $d-k$ parameters from the model update. Nonetheless, given the server's inability to differentiate between malicious and benign updates, the sparsification also introduces sparsification errors to the benign updates, diminishing the model's utility, especially when $k$ is small. Fortunately, the Top-$k$ sparsifier prunes a vector by keeping the largest elements, so that the sparsified vector still contains the core information of the original vector, resulting in a low sparsification error. Compared with other sparsification method like random sparsification, Top-$k$ sparsifier is much more robust when $k$ is small~\cite{shi2019layer,hu2023federated}. 
Moreover, individual sparsification of each model update can better preserve the model utility, compared with the uniform sparsification in existing works~\cite{VRDP, meng2023enhancing}. This is particularly beneficial in non-IID settings, where variability in local data distributions results in diverse local model updates. In addition, applying sparsification before aggregation allows the subsequent layer-wise filtering to concentrate on the key parameters critical to model performance, thereby eliminating interference from less important parameters. 

\textbf{Layer-wise adaptive aggregation.} It has been observed that different layers in the deep neural network differ in their sizes, functionality, and more importantly, learning and converging speed~\cite{rehman2023dawa, liu2023yoga, ma2022layer, lee2023layer}. However, most of the existing robust aggregation rules perform model-wise or coordinate-wise assessment on the local model updates, which usually fail to identify the nuances of each layer. Therefore, in LASA, we design a layer-wise filtering and aggregation after model sparsification, which enhances robustness with more precise, layer-specific granularity. 

In the context of Byzantine attacks, where attackers aim to deviate the global model's convergence in the wrong directions, malicious updates are typically crafted to significantly diverge from benign updates, both in magnitude and direction. Hence, the LASA, both magnitude and direction of each model update are captured at the layer level to effectively identify and filter out malicious clients. More precisely, for each layer $l\in[L]$ of the sparsified model update, its magnitude is quantified using the $L_2$-norm, and its direction is determined by analyzing the signs of its parameters. Inspired by~\cite{purity}, a direction metric, termed as \textit{Positive Direction Purity (PDP)} is defined in Definition~\ref{def:purity}. 
\begin{definition}[Positive Direction Purity]\label{def:purity}
    For a vector $x \in \mathbb{R}^d$, the positive direction purity $\rho$ of $x$ is defined as
\begin{equation}\label{eq:PDP}
    \rho := \frac{1}{2} \times \left(1+\frac{\sum_{i=1}^d sgn([x]_i)}{\sum_{i=1}^d |sgn([x]_i)|}\right), \ 0 \leq \rho \leq 1,
\end{equation}
where $sgn(\cdot)$ is the function to take the sign of each element and $[\cdot]_i$ is the $i$-th coordinate of a vector.
\end{definition}

PDP serves as a metric to evaluate the predominance of positive signs within a given vector, providing a refined approach for identifying anomalies in model direction. 
As a normalized measure, PDP measures a vector's overall orientation toward positive values, which is useful for analyzing directional tendencies. It is particularly effective in detecting stealthy attacks where the malicious models might not exhibit significant variations in magnitude. It is worth noting that the pre-aggregation sparsification can significantly enhance the PDP-based measurement of model direction by removing many less important parameters. This removal is particularly significant for PDP, which relies solely on the signs of the parameters, ensuring that the measurement focuses on the parameters with large values.

With the magnitude and direction metrics of a layer (line~\ref{lasa_l2}-\ref{lasa_pdp}), LASA will filter out clients that exhibit extreme values (either excessively high or low) using pre-defined thresholds. Given that the magnitude and direction values for each layer can vary significantly, setting such thresholds necessitates layer-specific customization. This, however, leads to a proliferation of hyper-parameters. Inspired by the traditional standardization method \textit{Z-score}, we introduce a robust variant named \textit{Median-based Z-score (MZ-score)}, as defined in Definition~\ref{def:mz}. 
\begin{definition}[MZ-score]\label{def:mz}
    For a set of values $X \coloneqq \{x_1, \dots, x_n \}$ with median $\textit{Med}(X)$ and standard deviation $\sigma$, the MZ-score $\lambda_i$ of $x_i \in X$ is defined as 
    \begin{equation}
    \lambda_i := \frac{x_i - \textit{Med}(X)}{\sigma}.
    \label{eq:balancescore}
\end{equation}
\end{definition}
This variant indicates how many standard deviations an element is from the median, which can be positive or negative. Importantly, MZ-score allows a uniform filtering radius to be applied across all layers, which substantially reduces the number of hyper-parameters required for adaptive layer-wise filtering. Specifically, in LASA, MZ-scores of magnitude and direction metrics are calculated at the layer level for all clients (line~\ref{lasa_mz_m}-\ref{lasa_mz_d}). Model updates with high absolute MZ-score values are then filtered out using two pre-defined filtering radiuses: $\lambda_m$ for magnitude and $\lambda_d$ for direction. Subsequently, the clients that remain, considered benign, are added to the set $\mathcal{S}$ and will participate in layer-wise model averaging (line~\ref{lasa_layerwiseagg}). 

\section{Robustness and Resilience Analysis of LASA} \label{sec:conver}
Before presenting our theoretical results, we make the following assumptions:

\begin{assumption}[$\mu$-Smoothness] \label{ass1}
    Each local objective function $\mathcal{L}_{i}$ for benign client $i\in\mathcal{B}$ is $\mu$-Lipschitz smooth with $\mu >0$, i.e., for any $x,y \in \mathbb{R}^d$,
    $ 
        \left \| \nabla \mathcal{L}_{i}(x) - \nabla \mathcal{L}_{ i}(y) \right\| \leq \mu \left \| x-y \right\|, \forall i\in\mathcal{B},
        $ 
    which further gives:
    $
        \mathcal{L}_{i}(x) - \mathcal{L}_{ i}(y) \leq \nabla \mathcal{L}_{ i}(x)^T(y-x) + \frac{\mu}{2}\left \| x-y \right \|^2, \forall i\in\mathcal{B}.
         $ 
\end{assumption}

\begin{assumption}[Unbiased gradient and bounded variance]\label{ass2}
    The stochastic gradient at each benign client is an unbiased estimator of the local gradient, i.e., $\mathbb{E}[g_i(x)] = \nabla \mathcal{L}_i(x)$ and has bounded variance, i.e., for any $ x\in\mathbb{R}^d$,
$ 
    \mathbb{E} \left\|g_i(x) - \nabla \mathcal{L}_{i}(x)) \right\|^2 \leq \nu_i^2, \forall i\in\mathcal{B},
    $ 
where the expectation is over the local mini-batches. We also denote
$\bar{\nu} \coloneqq \left(1/ |\mathcal{B}|\right)\sum_{i\in\mathcal{B}}\nu^2_i$ for convenience. 
\end{assumption}

\begin{assumption}[Bounded heterogeneity]\label{ass3} There exist a real value $\bar{\zeta}$ such that for any $x \in \mathbb{R}^d$,
$ 
    \frac{1}{|\mathcal{B}|}\sum_{i\in\mathcal{B}}\left \| \nabla \mathcal{L}_{i}(x) - \nabla \mathcal{L}_{\mathcal{B}}(x) \right \|^2 \leq \bar{\zeta},
    $ 
where the $\nabla \mathcal{L}_\mathcal{B}(x) \coloneqq \left(1/ |\mathcal{B}|\right)\sum_{i\in\mathcal{B}}\mathcal{L}_{i}(x)$.
\end{assumption}
Note that Assumption~\ref{ass1}-\ref{ass2} are commonly used in the theoretical analysis of distributed learning systems~\cite{hu2023federated, signguard, nesterov2018lectures}. While Assumption~\ref{ass3} states a standard measure of inter-client heterogeneity in FL~\cite{kappa, karimireddy2021byzantine, el2021collaborative}, which complicates the problem of Byzantine FL~\cite{kappa}. Note that these assumptions apply to benign clients only, as malicious clients do not follow the prescribed local training protocol of FL. 
\begin{assumption}[Bounded sparsification] \label{boundedsp} Given a vector $ x\in\mathbb{R}^d$, there exists non-negative constants $c_k\in[0, 1]$ and $b_k\in[0, 1]$, so that the Top-$k$ sparsifier in Definition~\ref{def:spar} satisfies  
    $ 
        \|\textit{Top}_k(x) \|^2 \leq c_k\| x\|^2,
        $ 
    and 
    $ 
        \|\textit{Top}_k(x) -x \|^2 \leq b_k \|x \|^2.
        $ 
\end{assumption}
As LASA incorporates model sparsification, we make the following Assumption~\ref{boundedsp} on the Top-$k$ sparsifier in Definition~\ref{def:spar}. This assumption applies for any $k\in[0,d]$ due to the fact that $\|x \|^2 = \|\textit{Top}_k(x) -x \|^2 + \|\textit{Top}_k(x) \|^2$. A smaller $k$ implies a higher degree of sparsification and yields a smaller $c_k$ and a larger $b_k$. 

\subsection{Robustness analysis of LASA} \label{sec:robustnessanlysis}

To theoretically evaluate the efficacy of LASA, we introduce the concept of $\kappa$-robustness in Definition~\ref{f_kappa_def}. Note that Definition~\ref{f_kappa_def} is similar to $(f, \kappa)$-robustness defined in~\cite{kappa, allouah2023robust}, $(\delta_{\max}, c)$-ARAgg defined in~\cite{karimireddy2021byzantine, gorbunov2022variance, malinovsky2023byzantine}, and $(f, \lambda)$-resilient averaging defined in~\cite{farhadkhani2022byzantine}. Our robustness definition adopts a constant upper bound and focuses on quantifying the distance between the output of a robust aggregation rule and the average of all benign updates, which represents the optimal output of such a rule. We denote the set of benign clients as $\mathcal{B}$ so that $\mathcal{B} \subseteq \mathcal{N}$, where $\mathcal{N}$ is the client set.

\begin{definition}[$\kappa$-robustness] \label{f_kappa_def} Let $n> 1$ and $0 \leq f < n/2$. An aggregation rule $F\colon \mathbb{R}^{d\times n} \rightarrow \mathbb{R}^d$ is $\kappa$-robust if for any vectors $ x_1, \dots, x_n \in \mathbb{R}^d$ and a benign set $\mathcal{B} \subseteq \mathcal{N}$ of size $n - f$, the output $\hat{x} \coloneqq F(x_1, \dots, x_n)$ satisfies
$
    \mathbb{E}\left \| \hat{x} - \bar{x}_\mathcal{B} \right \|^2 \leq \kappa,
    $ 
where $\bar{x}_\mathcal{B} \coloneqq \frac{1}{|\mathcal{B}|} \sum_{i\in \mathcal{B}} x_i$, $\kappa \geq 0$ refers to the \textit{robustness coefficient} of the aggregation rule $F$, and the expectation is taken over the randomness of the inputs. 
\end{definition}
The $\kappa$-robustness guarantees that the error of an aggregation rule, in estimating the average of the benign inputs, is upper-bounded by $\kappa$. With Definition~\ref{f_kappa_def}, we prove that when LASA is applied to $n$ input models, of which $f < n / 2$ are malicious, satisfies $\kappa$-robustness with $\kappa = {O}(c_k (1 + f / (n-2f)))$, as stated in Lemma~\ref{lemma:LASAkappa}. Note that LASA enjoys a higher robustness than several classical robust aggregation rules, for example, GeoMed~\cite{geomed} ($O(1 + f / (n-2f))^2$), and Krum~\cite{krum_mkrum_bulyan} ($O(1 + f / (n-2f))$)
\footnote{Results of GeoMed and Krum are taken from~\cite{kappa}. Note that the definition of $\kappa$ in~\cite{kappa} is different from ours, but since we are concerned with the \textit{order} of $\kappa$, we can safely incorporate these results into our discussion without losing generality.}.

\begin{lemma}[$\kappa$-robustness of LASA] \label{lemma:LASAkappa}
Under Assumption~\ref{ass1}-~\ref{boundedsp}, if $n \geq 1$ and  $0 \leq f < n/2$, the proposed LASA method is a $\kappa$-robust aggregation rule with
\begin{align*} 
    \kappa = 2c_k\left(1+ \frac{f}{n-2f}\right)(2\bar{\nu} + \bar{\zeta}+2C_{\lambda_m}^2+2C^2) + b_k C^2,
\end{align*}
if the learning rate $\eta\leq 1/2\tau$ and the selection set satisfies $|\mathcal{S}^l| \geq n/2 - f,\forall l \in [L]$, $\tau$ is the number of local iteration. Here, $C_{\lambda_m}^2$ and $C^2$ represent the upper bound of the norm of malicious and benign updates, respectively. 

If the sparsification parameter $k$ satisfies that $c_k \leq 1 / (1 + \epsilon)$ and $(b_k / c_k) \leq \epsilon (4+(4\bar{\nu} +2 \bar{\zeta}+4C_{\lambda_m}^2) / C^2 )$ with a positive constant $\epsilon$,
we have 
 $ 
\kappa =O\left(c_k\left(1+ \frac{f}{n-2f}\right)\right).
 $ 

\begin{proof}
The detailed proof is given in Appendix~\ref{proofofpro1}.
\end{proof}
\begin{remark}
\textbf{(1) Extreme cases.} Extremely, without sparsification, i.e., $k=d$, we have $c_k=1$ and $b_k=0$. In this case, the robustness upper bound of LASA, denoted by $\kappa_1$, is 
 $ 
    \kappa_1 = 2\left(1+ \frac{f}{n-2f}\right)(2\bar{\nu} + \bar{\zeta}+2C_{\lambda_m}^2+2C^2).  
    $ 
When $k=0$, we have $c_k=1$ and $b_k=0$ which gives an upper bound of $C^2$, indicating the greatest sparsification error to robustness. \textbf{(2) Proper $k$ yields higher robustness.} When $0< k < d$, if the sparsification parameter $k$ is selected to satisfy the two conditions in Lemma~\ref{lemma:LASAkappa}, the robustness upper bound in this case, denoted by $\kappa_2$, is 
$\kappa_2= (1+\epsilon)c_k\kappa_1$. 
As $c_k\leq 1/(1+\epsilon)$, we have $\kappa_2 \leq \kappa_1$, which demonstrates the effectiveness of sparsification in amplifying robustness. Moreover, the conditions on $k$ indicate that when the local divergence/variance or the magnitude of the malicious model is large, we can select a relatively small $k$ (which gives a large $b_k$ and small $c_k$) to amplify the robustness. Hence, the benefit of sparsification will be more significant in non-IID settings theoretically. \textbf{(3) Impact of $C_{\lambda_m}^2$ and $C^2$.} Lemma~\ref{lemma:LASAkappa} also shows that theoretically the larger the magnitude of benign updates or malicious updates is, the lower the robustness will be. Indeed, in the literature, model clipping has demonstrated its effectiveness in mitigating the impact of malicious~\cite{zhang2019gradient, panda2022sparsefed, signguard}. In LASA, the norm of malicious updates is particularly bounded by the magnitude-based filtering which is controlled by the hyper-parameter $\lambda_m$ in Algorithm~\ref{alg:main}. A smaller $\lambda_m$ indicates a relatively smaller $C_{\lambda_m}^2$. 
\textbf{(4) Impact of $\bar{\nu}$ and $\bar{\zeta}$.} Note that $\bar{\nu}$ and $\bar{\zeta}$ represent the local variance and local heterogeneity in Assumption~\ref{ass2}-\ref{ass3}. The findings in Lemma~\ref{lemma:LASAkappa} highlight the importance of reducing the variance of stochastic gradient and mitigating the divergence due to non-IID data distribution to enhance robustness. Our work is orthogonal to existing variance or divergence reduction methods~\cite{gorbunov2022variance} and can be combined with them to further improve the robustness. 

\end{remark}

\end{lemma}

\subsection{Resilience analysis of FL with LASA}

Similar to~\cite{kappa, allouah2024robust}, we define Byzantine resilience of a FL algorithm in Definition~\ref{def_of_byz_resi} as follows.
\begin{definition}[$(f, R)$-Byzantine resilience~\cite{kappa, allouah2024robust}] \label{def_of_byz_resi} With the presence of $f$ Byzantine clients, a FL algorithm is said $(f, R)$-Byzantine resilient if it outputs $\widetilde{\theta}$ such that
    $ 
        \mathbb{E}\left \| \nabla\mathcal{L}_\mathcal{B}(\widetilde{\theta}) \right \|^2 \leq R,
        $ 
    where $\mathcal{B}$ denotes the set of benign clients, $\mathcal{L}_\mathcal{B}(\theta) \coloneqq \frac{1}{|\mathcal{B}|}\sum_{i \in \mathcal{B}} \mathcal{L}_i(\theta)$, and expectation is taken over the randomness of the FL algorithm.
\end{definition}

In words, a $(f, R)$-Byzantine resilient FL finds an $R$-\textit{approximate} stationary point for the honest loss, despite the presence of up to $f$ Byzantine clients. This definition is crucial as it quantifies the level of tolerance a Byzantine-resilient FL has against the potentially harmful influence of Byzantine clients. Note that $f$ is assumed to be less than $n/2$, as it is generally impossible for an FL algorithm with $F$ to achieve Byzantine resilience when $f\geq n/2$~\cite{liu2021approximate}. 

Now we prove that FL with LASA is $(f, R)$-Byzantine resilient and achieves the asymptotic error bounded by $R$ in the presence of $f$ Byzantine clients, as stated Theorem~\ref{lasa_convergence}.

\begin{theorem}[$(f, R)$-Byzantine resilience of LASA] \label{lasa_convergence} Let $\theta^0$ be the initial point and $\theta^*$ be the optimal point. Assume $\widetilde{\theta}$ is uniformly sampled from the sequence of outputs $\{\theta^0, \theta^1, \dots, \theta^T \}$ generated by FL with LASA. Under Assumption~\ref{ass1}-\ref{boundedsp}, suppose the learning rate $\eta$ satisfies $ \eta \leq \min\{1/2\tau, 1/3\mu \tau\}$, then we have
    \begin{align*}
    \mathbb{E}\left \| \nabla\mathcal{L}_\mathcal{B}(\widetilde{\theta}) \right \|^2 &\leq \frac{ \mathcal{L}_\mathcal{B}(\theta^{0}) - \mathcal{L}_\mathcal{B}(\theta^{*}) }{T\eta}  + \kappa\left(\mu \eta + 1 \right) \nonumber \\
    & \quad \ + 
         7  \tau \bar{\zeta} + (1 + \tau)   \bar{\nu},
\end{align*}
where $\kappa$ represents the robustness coefficient of LASA.
\begin{proof}
    The detailed proof is given in Appendix~\ref{proofofthe1}. 
\end{proof}
\begin{remark}
The last two terms, i.e., $7 \tau \bar{\zeta}+(1 + \tau) \bar{\nu}$, represent the convergence errors due to data heterogeneity and gradient variance and will be eliminated when local data are IID and full gradient is calculated. The second term represents the \textit{Byzantine error} associated with the robustness coefficient $\kappa$. 
Note that due to the client sampling in FL, $h$ out of $n$ clients are selected uniformly at random to participate in the training per round, and the expected number of malicious clients is $hf/n$ per round, which does not affect the expected value of $\kappa$. Recall that selecting an appropriate sparsification parameter $k$ allows LASA to achieve a smaller $\kappa$, leading to a smaller convergence error. 
However, choosing an unsuitable $k$, such as when $k\ll d$, may result in a very high sparsification error, which will dominate $\kappa$ and make the robustness amplification benefit of sparsification negligible. This finally leads to a higher convergence error and lower Byzantine resilience. We discuss the selection of $k$ in Lemma~\ref{lemma:LASAkappa} and also study its impact on model performance during the evaluation.  
\end{remark}
\end{theorem}


\begin{table*}[t]
  \centering
  \caption{Testing Accuracy (\%) of Different Defense Methods in IID Settings.
  }
   \vspace{-5pt}
  \scalebox{0.73}{
    \begin{tabular}{cc|c|cccccccc|c}
    \toprule
    \multicolumn{1}{c}{\multirow{1}[4]{*}{\textbf{\makecell*[c]{Dataset \\ (Model) }}}} & \multirow{1}[4]{*}{\textbf{\makecell*[c]{Defense \\ Method }}} &  \multicolumn{1}{c}{\multirow{1}[4]{*}{\textbf{\makecell*[c]{No \\ Attack }}}} & \multicolumn{3}{c}{\textbf{Naive Attacks}} & \multicolumn{5}{c}{\textbf{SOTA Attacks}} &\multicolumn{1}{c}{\multirow{1}[4]{*}{\textbf{\makecell*[c]{Average \\ w/ Attacks }}}}\\
\cmidrule(r){4-6}  \cmidrule(r){7-11}        &       &       & Random & Noise & Sign-flip & TailoredTrMean & Min-Max & Min-Sum & Lie   & ByzMean \\
    \midrule
    \multicolumn{1}{c}{\multirow{9}[2]{*}{\makecell*[c]{FMNIST \\ (CNN) }}} & FedAvg  & 86.28 & 29.20 & 41.66 & 83.91 & 10.08 & 77.89 & 79.56 & 83.47 & 11.22 & 52.12 \\
          & TrMean & 84.05 & 80.56 & 81.26 & 81.11 & 10.59 & 69.95 & 73.01 & 77.78 & 10.54 & 59.85 \\
          & GeoMed & 84.10 & 84.30 & 84.30 & 82.28 & 84.51 & 60.86 & 50.32 & 65.08 & 83.05 & 74.34 \\
          & Multi-Krum & 86.91 & 84.15 & 84.33 & 85.34 & 10.00 & 68.76 & 80.67 & 80.43 & 11.80 & 63.19 \\
          & Bulyan & 81.35 & 83.81 & 83.84 & 78.59 & 33.76 & 59.24 & 62.09 & 73.29 & 59.75 & 66.80 \\
          & DnC   & 87.30 & 84.23 & 84.17 & 85.69 & 32.66 & 69.81 & 79.08 & 81.96 & 63.11 & 72.59 \\
          & SignGuard & \textbf{87.63} & \underline{87.72} & \underline{87.72} & \underline{87.06} & \underline{87.40} & \underline{87.40} & \underline{87.18} & \underline{87.17} & \underline{87.61} & \underline{87.41} \\
          & SparseFed & 86.27 & 29.48 & 41.10 & 83.86 & 10.08 & 77.88 & 79.55 & 83.47 & 11.28 & 52.09 \\
          & \cellcolor[rgb]{ .816,  .808,  .808}LASA (Ours)  & \cellcolor[rgb]{ .816,  .808,  .808}\underline{87.62} & \cellcolor[rgb]{ .816,  .808,  .808}\textbf{87.92} & \cellcolor[rgb]{ .816,  .808,  .808}\textbf{87.87} & \cellcolor[rgb]{ .816,  .808,  .808}\textbf{87.13} & \cellcolor[rgb]{ .816,  .808,  .808}\textbf{87.97} & \cellcolor[rgb]{ .816,  .808,  .808}\textbf{87.91} & \cellcolor[rgb]{ .816,  .808,  .808}\textbf{87.36} & \cellcolor[rgb]{ .816,  .808,  .808}\textbf{87.54} & \cellcolor[rgb]{ .816,  .808,  .808}\textbf{87.65} & \cellcolor[rgb]{ .816,  .808,  .808}\textbf{87.67} \\
    \midrule
    \vspace{-14pt}
          &       &       &       &       &       &       &       &       &       &  \\
    \midrule
    \multicolumn{1}{c}{\multirow{9}[2]{*}{\makecell*[c]{CIFAR-10 \\ (ResNet18~\cite{resnet}) }}} & FedAvg   & 89.70 & 44.34 & 47.65 & 82.07 & 15.77 & 76.26 & 61.25 & 84.76 &   13.01 & 53.14\\
          & TrMean   & \textbf{90.14} & 87.36 & 87.36 & 84.77 & 49.65 & 61.30 & 57.58 & 77.40 &  49.61  & 69.38\\
          & GeoMed   & \underline{89.85} & \underline{87.76} & \underline{87.57} & 85.74 & 71.22 & 63.42 & 71.91 & 70.78 &  87.45  & 78.23\\
          & Multi-Krum   & 84.73 & 84.62 & 84.72 & 84.58 & 84.49 & 47.97 & 53.16 & 44.26 & 84.55  & 71.04\\
          & Bulyan   & 88.97 & 87.68 & 87.56 & \underline{86.52} & 85.03 & 38.38 & 47.29  & 53.30 &  84.96  & 71.34\\
          & DnC     & 89.54 & 59.26 & 61.33 & 84.72 & 38.75 & 63.34 & 61.11 & 67.30 &  57.08 & 61.61\\
          & SignGuard   & 89.47 & 82.36 & 81.68 & 80.04 & \underline{88.45} & \underline{88.14} & \underline{88.09} & \underline{88.11} & \underline{88.39}   & \underline{85.66}\\
          & SparseFed   & 89.65 & 43.93 & 48.22 & 82.18 & 15.92 & 75.90 & 68.13 & 84.91 &  10.00  & 53.65\\
          & \cellcolor[rgb]{ .816,  .808,  .808} LASA (Ours)    &\cellcolor[rgb]{ .816,  .808,  .808}89.00& \cellcolor[rgb]{ .816,  .808,  .808}\textbf{88.66}& \cellcolor[rgb]{ .816,  .808,  .808}\textbf{88.81}& \cellcolor[rgb]{ .816,  .808,  .808}\textbf{86.68}& \cellcolor[rgb]{ .816,  .808,  .808}\textbf{89.05}& \cellcolor[rgb]{ .816,  .808,  .808}\textbf{88.57}& \cellcolor[rgb]{ .816,  .808,  .808}\textbf{88.96}& \cellcolor[rgb]{ .816,  .808,  .808}\textbf{88.59}& \cellcolor[rgb]{ .816,  .808,  .808}\textbf{89.08}& \cellcolor[rgb]{ .816,  .808,  .808}\textbf{88.55} \\
            \midrule
    \vspace{-14pt}
          &       &       &       &       &       &       &       &       &       &  \\
    \midrule
    \multicolumn{1}{c}{\multirow{9}[2]{*}{\makecell*[c]{CIFAR-100 \\ (ResNet18~\cite{resnet}) }}}  
          & FedAvg   & \underline{65.98} & 12.09 & 14.20 & 48.05 & 1.96 & 45.90 & 44.26 & 57.66 & 1.36 & 28.19\\
          & TrMean   & 65.40 & 63.19 & 63.22 & 52.65 & 28.57 & 34.80 & 33.83 & 48.92 & 30.11 & 44.41\\
          & GeoMed   & 65.84 & \underline{63.33} & \textbf{63.57} & \underline{61.25} & 39.78 & 38.95 & 39.34 & 46.93 & 61.29 & 51.81\\
          & Multi-Krum  & 52.71 & 52.14 & 52.24 & 52.92 & 52.88 & 19.03 & 19.82 & 24.43 & 53.26 & 40.84\\
          & Bulyan   & 61.29 & 61.06 & 61.32 & 60.01 & 50.84 & 17.83 & 19.17 & 30.27 & 56.75 & 44.66\\
          & DnC     & 65.53 & 24.47 & 28.05 & 53.39 & 11.21 & 29.37 & 28.51 & 34.04 & 29.05 & 29.76\\
          & SignGuard   & 65.64 & 63.24 & 63.36 & 47.55 & \underline{63.36} & \underline{63.20} & \underline{63.22} & \underline{62.72}  & \underline{62.94} & \underline{61.20}\\
          & SparseFed   & \textbf{65.99} & 12.06 & 14.06 & 48.01 & 1.83 & 46.05 & 44.34 & 57.74 & 1.53 & 28.20\\
          & \cellcolor[rgb]{ .816,  .808,  .808} LASA (Ours)    &65.52 \cellcolor[rgb]{ .816,  .808,  .808}&\textbf{63.48} \cellcolor[rgb]{ .816,  .808,  .808}&\underline{63.49} \cellcolor[rgb]{ .816,  .808,  .808}&\textbf{62.89} \cellcolor[rgb]{ .816,  .808,  .808}&\textbf{63.71} \cellcolor[rgb]{ .816,  .808,  .808}&\textbf{63.54} \cellcolor[rgb]{ .816,  .808,  .808}&\textbf{63.63} \cellcolor[rgb]{ .816,  .808,  .808}&\textbf{63.98} \cellcolor[rgb]{ .816,  .808,  .808}& \cellcolor[rgb]{ .816,  .808,  .808}\textbf{63.85}& \cellcolor[rgb]{ .816,  .808,  .808}\textbf{63.57} \\
    \bottomrule
    \end{tabular}%
    }
  \label{tab:iid_results}%
  \vspace{-10pt}
\end{table*}%
\section{Evaluation} \label{sec:eva}

\textbf{Experimental settings.} 
To comprehensively demonstrate the effectiveness of LASA, we compare it with the non-robust baseline \textit{FedAvg} and seven existing SOTA defense methods, including \textit{Bulyan}~\cite{bulyan}, \textit{Trimmed Mean (TrMean)}~\cite{trmean}, \textit{Geometric median (GeoMed)}~\cite{geomed}, \textit{Multi-Krum}~\cite{krum_mkrum_bulyan}, \textit{Divide-and-Conquer (DnC)}~\cite{dnc}, \textit{SignGuard}~\cite{signguard}, and \textit{SparseFed}~\cite{panda2022sparsefed}. We test three naive attacks including \textit{Random}, \textit{Noise} and \textit{Sign-flip} attacks and five SOTA attack methods including \textit{Min-Max}~\cite{dnc}, \textit{Min-Sum}~\cite{dnc}, \textit{AGR-tailored Trimmed-mean}~\cite{dnc}, \textit{Lie}~\cite{lie}, and \textit{ByzMean}~\cite{signguard} attacks. We mainly conduct experiments on \textit{FMNIST}~\cite{FMNIST}, \textit{FEMNIST}~\cite{leaf}, \textit{CIFAR-10}~\cite{cifar10}, \textit{CIFAR-100}~\cite{cifar10} and \textit{Shakespeare}~\cite{FL_OG_shake} datasets. FEMNIST and Shakespeare datasets are naturally non-IID. 
{For FMNIST, CIFAR-10 and CIFAR-100 datasets, we evenly split the dataset over the clients to simulate the IID settings, and use \textit{Dirichlet distribution}~\cite{minka2000estimating} $Dir(\alpha)$ to simulate the non-IID settings with a default non-IID degree $\alpha=0.5$.} 
The default \textit{attack ratio} is set to 25\%, meaning 25\% of the clients are malicious in our FL system. Note that the number of malicious clients selected for training per round may differ due to client sampling. For the hyperparameters in LASA, for all datasets, we set the \textit{sparsification level} to 0.3 (i.e., $1-k/d=0.3$), $\lambda_d$ is set to 1.0 by default while $\lambda_m$ is set to 1.0 for CIFAR10/100 and 2.0 for others. More details of the experimental settings are given in Appendix~\ref{more_experimental_setting}. Additionally, the attack and defense models are presented in Appendix~\ref{attack_defense_models}. 
{We run each experiment with three random seeds and report the average testing accuracy.} We use \textbf{bold font} to highlight the best results, while the second-best results are \underline{underlined}.

\textbf{Performance of LASA in IID settings.} We first comprehensively evaluate the performance of all the defense methods in IID settings. From the results on FMNIST, CIFAR-10, and CIFAR-100 datasets (shown in Table~\ref{tab:iid_results}), we can see that except for Noise attack on CIFAR-100 where LASA achieves second-best performance, LASA achieves the \textit{\textbf{highest accuracy under all attacks}}, outperforming other defense methods.
For example, under ByzMean attack, LASA, GeoMed, and SignGuard stand out as the most effective defense methods on FMNIST, and LASA achieves the highest accuracy of 87.65\%, which is +0.03\% and +4.60\% higher than SignGuard and GeoMed, respectively. On CIFAR-10 under TailoredTrMean attack, LASA achieves the highest accuracy of 89.05\%, which is +0.60\% and +4.02\% higher than SignGuard and Bulyan, respectively. 
\textit{It is noteworthy that \textbf{LASA reaches ``ceiling'' performance levels} as it consistently achieves accuracy similar to scenarios without attacks.}
We observe that the classic robust aggregation rule-based methods including TrMean, GeoMed, Multi-Krum, and Bulyan, as well as the sparsification-based method SparseFed, fail to defend against advanced distance-based attacks like TailoredTrMean and ByzMean, which maximize or minimize the distance between benign and malicious models. The reason is that these classic robust aggregation rules often filter malicious parameters at coordinate level or based on model-wise distance; SparseFed is ineffective due to its limited capability in removing malicious parameters. The robustness of LASA, illustrated by the above-mentioned results, emphasizes its potential as a robust defense method in securing FL environments against a wide collection of attacks, ultimately enhancing the reliability of FL systems. Additional results on MNIST, FEMNIST, and Shakespeare datasets are given in Appendix~\ref{more_results}.

\begin{table}[t]
  \centering
  \caption{Testing Accuracy (\%) of LASA in Non-IID Settings on CIFAR-10 (C-10) and CIFAR-100 (C-100) Datasets, Compared with Multi-Krum, GeoMed and SignGuard under ByzMean.}
  \scalebox{0.66}{
    \begin{tabular}{c|c|cccccc|c}
    \toprule
    \multirow{2}[4]{*}{\textbf{Dataset}} & \multirow{2}[4]{*}{\textbf{Method}} & \multicolumn{6}{c}{\textbf{Non-IID degrees} $\alpha$} & \multirow{2}[4]{*}{\textbf{Avg.}} \\
\cmidrule{3-8}          &       & 0.1   & 0.2   & 0.3   & 0.4   & 0.5   & 1.0 \\
    \midrule
    \multirow{4}[2]{*}{C-10} 
    & Multi-Krum & 26.34 & 39.34 & 52.06 & 55.52 & 61.31 & 74.75 & 51.55\\
    & GeoMed & 49.14 & \underline{63.33} & \underline{71.45} & 72.82 & 75.72 & \underline{83.49} & 69.33\\
          & SignGuard & \underline{51.90} & 63.17 & 70.77 & \underline{75.05} & \underline{76.16} & 83.31 & \underline{70.06}\\
          & \cellcolor[rgb]{ .816,  .808,  .808}LASA (Ours) & \cellcolor[rgb]{ .816,  .808,  .808}\textbf{56.01} & \cellcolor[rgb]{ .816,  .808,  .808}\textbf{65.61} & \cellcolor[rgb]{ .816,  .808,  .808}\textbf{74.59} & \cellcolor[rgb]{ .816,  .808,  .808}\textbf{75.40} & \cellcolor[rgb]{ .816,  .808,  .808}\textbf{77.73} & \cellcolor[rgb]{ .816,  .808,  .808}\textbf{84.34} & \cellcolor[rgb]{ .816,  .808,  .808}\textbf{72.28}\\
    \midrule
    \vspace{-14pt}
          &       &       &       &       &       &       &  \\
    \midrule
    \multirow{4}[2]{*}{C-100} 
    & Multi-Krum & 25.83 & 37.91 & 43.64 & 45.58 & 47.58 & 51.37 & 41.99 \\
    & GeoMed & 46.46 & 53.80 & 56.99 & 58.16 & 58.72 & 59.79 & 55.65\\
          & SignGuard & \underline{48.95} & \underline{56.74} & \underline{58.50} & \underline{58.51} & \underline{59.82} & \underline{61.42} & \underline{57.32} \\
          & \cellcolor[rgb]{ .816,  .808,  .808}LASA (Ours) & \cellcolor[rgb]{ .816,  .808,  .808}\textbf{51.25} & \cellcolor[rgb]{ .816,  .808,  .808}\textbf{57.59} & \cellcolor[rgb]{ .816,  .808,  .808}\textbf{59.73} & \cellcolor[rgb]{ .816,  .808,  .808}\textbf{60.41} & \cellcolor[rgb]{ .816,  .808,  .808}\textbf{60.24} & \cellcolor[rgb]{ .816,  .808,  .808}\textbf{61.61} & \cellcolor[rgb]{ .816,  .808,  .808}\textbf{58.47}\\
    \bottomrule
    \end{tabular}%
    }
  \label{tab:non_iid}%
\vspace{-16pt}
\end{table}%
\textbf{Performance of LASA in various non-IID settings.} Here, we evaluate the effectiveness of LASA in various non-IID settings. We simulate different non-IIDness by varying the non-IID degree $\alpha$ from 0.1 to 1.0, where a smaller $\alpha$ indicates a more intense non-IIDness. 
From the results on CIFAR-10 and CIFAR-100 datasets under the SOTA ByzMean attack (shown in Table~\ref{tab:non_iid}), we observe that as $\alpha$ increases, the performance of all the defense methods improves due to the decreased data heterogeneity. Among them, LASA always achieves the highest accuracy under various non-IID degrees, leading a average of +1.15\% and +2.22\% over SOTA SignGuard. LASA individually sparsifies model updates thus reducing sparsification error, especially in non-IID cases with heterogeneous model updates. It also performs layer-wise filtering, allowing precise identification of benign/malicious model updates at a finer granularity. By adeptly integrating pre-aggregation sparsification and layer-wise adaptive aggregation, LASA effectively mitigates the impact of divergent updates, resulting in the highest accuracy among its counterparts.

\textbf{Effectiveness of LASA in model update identification.} 
\begin{figure}[t]
    \centering
    \includegraphics[width=0.90\linewidth]{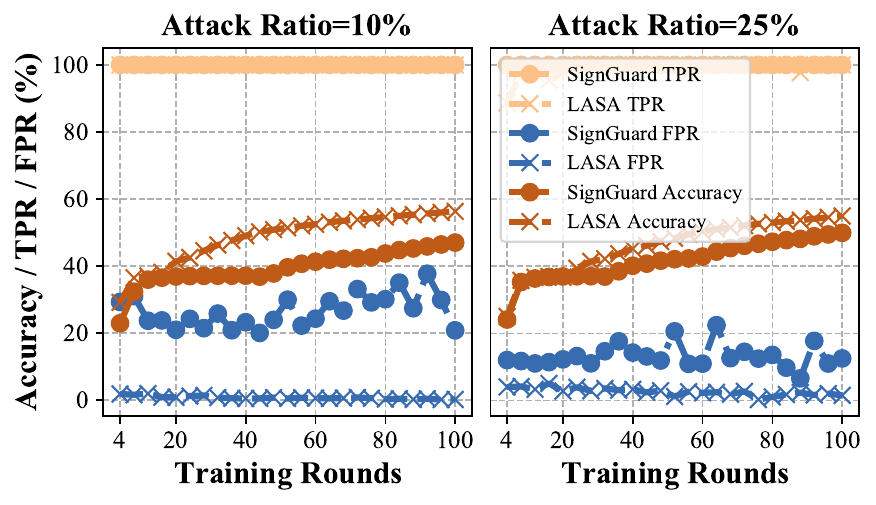}
    \vspace{-5pt}
    \caption{TPR, FPR, and Testing Accuracy (\%) of LASA and SignGuard under ByzMean Attack on Shakespeare Dataset.}
    \label{fig:CSR_WSR}
    \vspace{-15pt}
\end{figure}
To deeply investigate the effectiveness of LASA, we observe the behavior of LASA in identifying malicious updates, compared with the SOTA method SignGuard. Specifically, we use two metrics, \textit{True Positive Rate} (TPR) and \textit{False Positive Rate} (FPR), to evaluate their performance in identifying malicious updates and benign updates. A higher TPR and lower FPR imply a more accurate benign/malicious update identification. 
As shown in Figure~\ref{fig:CSR_WSR}, LASA and SignGuard achieve significantly high TPRs, which means they can effectively identify malicious updates. However, SignGuard achieves relatively high FPRs in both attack ratio settings. For example, with a high attack ratio of 25\%, at each round, SignGuard misidentifies about 15\% percent of benign updates as malicious ones, so that the convergence rate of the global model drops due to the lack of benign updates. In contrast, LASA always keeps a very low FPR, demonstrating the superior performance of our unique design of layer-wise adaptive filtering.

\textbf{Impact of various attack ratios.}
\begin{figure}[t]
    \centering
\includegraphics[width=0.90\linewidth]{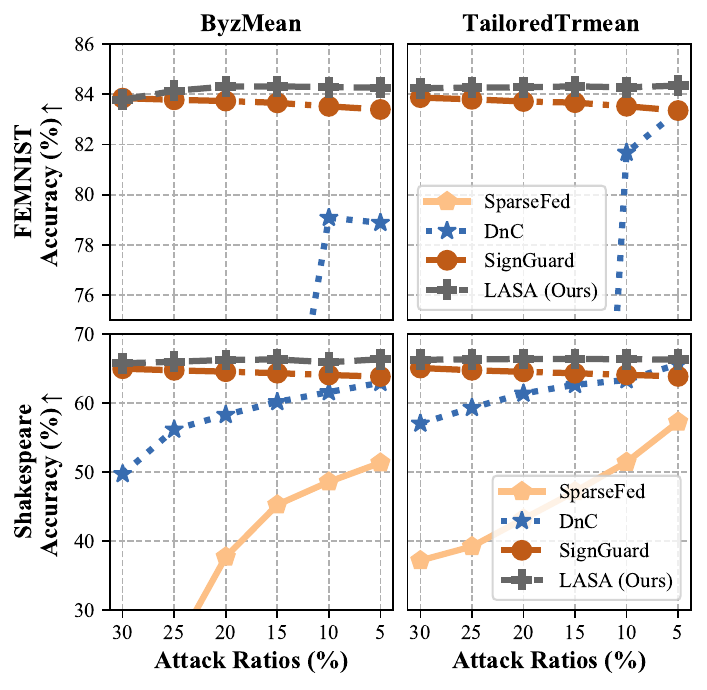}
    \vspace{-5pt}
    \caption{Testing Accuracy (\%) of LASA SparseFed, DnC, and SignGuard under Various Attack Ratios on the Non-IID FEMNIST (upper) and Shakespeare (lower) Datasets.}
    \label{fig:attackraio}
  \vspace{-15pt}
\end{figure}
We additionally evaluate the performance of three SOTA defense methods including DnC, SignGuard, and SparseFed, and our method LASA under different attack ratios on non-IID datasets and report the results in Figure~\ref{fig:attackraio}. Specifically, we conduct experiments under ByzMean and TailoredTrMean attacks with the attack ratio varying from 5\% to 30\%. In general, DnC and SparseFed's accuracies increase as the attack ratio decreases, but they suffer from significant accuracy degradation when the attack ratio is high. For instance, on FEMNIST, even when the attack ratio is as low as 5\%, SparseFed does not improve the robustness, achieving an accuracy of 7.44\% under the ByzMean attack. Similarly, DnC struggles to defend against ByzMean attack effectively until the attack ratio is reduced to 10\%, achieving a relatively low accuracy of 79.09\%. 
Compared to SignGuard, LASA achieves a better and more stable performance. As the attack ratio increases, LASA only has a minor decrease in accuracy. More results under other attacks with various attack ratios are given in Appendix~\ref{more_attack_ratios_results}.

\textbf{Ablation study.} As LASA consists of three key components to achieve Byzantine resilience, we conduct an ablation study to investigate how each component functions. We denote \textit{pre-aggregation sparsification}, \textit{magnitude-based adaptive filtering}, and \textit{direction-based adaptive filtering} as \textbf{\textit{Spar}}, \textbf{\textit{Ma}}, and \textbf{\textit{Di}}, respectively. 
Experimental results on CIFAR-10 and FEMNIST are summarized in Table~\ref{tab:ab}. As expected, only applying pre-aggregation sparsification does not provide enough robustness compared to LASA as it only removes partial less important malicious parameters. We observe that direction-based adaptive filtering is powerful in defending against the stealthy Min-Max and Lie attacks, but it is vulnerable to the simple Noise attack that generates malicious updates with large magnitudes. In contrast, magnitude-based adaptive filtering is effective in defending against Noise attack but is less effective against Min-Max and Lie attacks. Notably, when integrated with pre-aggregation sparsification, the performance of direction-based adaptive filtering improves: the accuracy of Spar+Di is higher than that of Di. This demonstrates the effectiveness of sparsification in improving the filtering accuracy. While LASA demonstrates comparable performance to Ma+Di, we emphasize in Section~\ref{sec:robustnessanlysis} the theoretical significance of Spar in enhancing robustness, thereby underscoring its necessity.
We further discuss the impact of different radii and sparsification levels in Appendix~\ref{sec:discuss_of_k}-~\ref{disscuss_of_lambda}. 

\begin{table}[t]
  \centering
  \caption{Comparison of Performance with Different Components}
  \scalebox{0.69}{
    \begin{tabular}{c|cccccc|c}
    \toprule
    \multirow{2}[4]{*}{\textbf{Method}} & \multicolumn{3}{c}{\textbf{CIFAR-10 (IID)}} & \multicolumn{3}{c}{\textbf{FEMNIST (non-IID)}} &\multicolumn{1}{c}{\multirow{2}[4]{*}{\textbf{Avg.}}} \\
    \cmidrule(r){2-4} \cmidrule(r){5-7}          
    & Min-Max & Lie   & Noise & Min-Max & Lie   & Noise \\
    \midrule
    Spar   & 75.00 & 87.00 & 86.04 & {44.39} & {81.08} & {54.91} & 71.40 \\
    Ma   & 85.73 & 87.99 & \textbf{{89.39}} & {40.45} & {79.38} & \textbf{{84.27}} & 77.87 \\
    Di   & \underline{{89.23}} & \underline{89.07} & {83.07} & \underline{84.26} & \textbf{{84.28}} & {68.40} & 83.05 \\
    Spar+Ma   & {84.66} & {87.99} & \underline{89.38} & {36.33} & {79.43} & \underline{84.26} & 77.01 \\
    Spar+Di   & \textbf{{89.28}} & \textbf{{89.18}} & {83.87} & \textbf{{84.28}} & \underline{84.26} & {69.03} & 83.32 \\
    Ma+Di   & {88.35} & {88.38} & {88.51} & {84.20} & {83.52} & {83.97} & \underline{86.16} \\
    \cellcolor[rgb]{ .816,  .808,  .808}LASA (Ours)   & \cellcolor[rgb]{ .816,  .808,  .808}{88.57} & \cellcolor[rgb]{ .816,  .808,  .808}{88.59} & \cellcolor[rgb]{ .816,  .808,  .808}{88.81} & \cellcolor[rgb]{ .816,  .808,  .808}{84.19} & \cellcolor[rgb]{ .816,  .808,  .808}{83.52} & \cellcolor[rgb]{ .816,  .808,  .808}{84.05} & \cellcolor[rgb]{ .816,  .808,  .808}\textbf{86.29}\\
    \bottomrule
    \end{tabular}%
    }
  \label{tab:ab}%
\vspace{-14pt}
\end{table}%

\section{Conclusion}~\label{sec:conclusion}
We present a novel Byzantine-resilient aggregation rule called LASA. LASA combines a pre-aggregation sparsification that sparsifies each local model update before aggregation with a novel layer-wise adaptive aggregation that filters and aggregates the sparsified model updates based on the magnitude and direction of each model layer. 
We theoretically analyze the robustness of LASA and provide the resilience analysis results of FL with LASA and then conduct extensive experiments on both IID and non-IID datasets to evaluate the effectiveness of LASA. Experimental results demonstrate that LASA outperforms other defense methods under both naive and advanced attacks. 

{\small
\bibliographystyle{ieee_fullname}
\bibliography{egbib}
}

\clearpage
\section{Appdenix}
\subsection{Attack and defense models}
\label{attack_defense_models}
\textbf{Attack model.} We follow the attack model in previous works~\cite{fang2020local, lie, dnc, signguard}. Specifically, the attacker controls a subset of $f$ malicious clients within the FL system. These clients can either be fake clients injected into the system by the attacker or genuine clients that have been compromised. The goal of the attacker is to degrade the overall performance of the global model in FL. The attacker has full knowledge of all benign updates in each training round. For additional background knowledge of the attacker, we follow the same settings of the proposed attack works. The malicious clients need not follow the prescribed local training protocol of FL and may send arbitrary local model updates to the server. Let $\mathcal{B}$ denote the set of benign clients in the system so that $\mathcal{B} \subset \mathcal{N}$. Under the Byzantine attack, the local model update of a client $i\in\mathcal{N}$ can be represented as
\begin{align}
    \Delta_i = \begin{cases}
            \Delta_i, & \text{if} \ i \in \mathcal{B}\\
            \beta_i, & \text{if} \ i \notin \mathcal{B}
        \end{cases}
\end{align}
where $\beta_i\in\mathbb{R}^d$ represents an arbitrary model depending on the specific attack method. 

\textbf{Defense goal.} Like previous works~\cite{fang2020local, cao2020fltrust, signguard}, we assume the server to be the defender who can deploy a robust aggregation rule, denoted by $F$, to mitigate the negative impact of malicious local models on the global model. The server has full access to the global model and local model updates in each training round, but it does not have access to the local training data of clients. We assume the server does not know the number of malicious clients unless explicitly specified. In addition, we assume that clients' submissions are made anonymously so that the server cannot track clients' actions.

\subsection{Experimental settings}
\label{more_experimental_setting}
We utilize six benchmark datasets of FL, including MNIST~\cite{MNIST}, Fashion-MNIST~\cite{FMNIST}, FEMNIST~\cite{leaf}, CIFAR-10~\cite{cifar10}, CIFAR-100~\cite{cifar10}, and Shakespeare~\cite{FL_OG_shake} datasets, to conduct the performance evaluation. 
The MNIST dataset is composed of gray-scale images of size $28\times28$ pixels for image classification tasks. It has 60,000 images for training and 10,000 images for testing. %
Similar to MNIST, Fashion-MNIST (FMNIST) dataset contains 70,000 $28\times28$ grayscale images for 10 categories of fashion products. The dataset is divided into 60,000 training images and 10,000 test images. For MNIST and FMNIST datasets, we evenly split the training data over 6,000 clients so that the distribution of private datasets on each client is IID.  
The Federated Extended MNIST (FEMNIST) dataset is a non-IID FL dataset extended from MNIST. It consists of 805,263 images hand-written by 3,550 users for a total of 62 classes, including 52 for upper and lower case characters and 10 for digits. 
We subsample 5\% of the original data following~\cite{leaf}, resulting in 1,827 clients with a total of 450,632 images. The number of samples for each client ranges from 3 to 525. 
The Shakespeare dataset is naturally a non-IID FL dataset for the next character prediction tasks. Following ~\cite{reddi2020adaptive}, we process the original data and result in a dataset consisting of 37,784 samples from 715 clients. 

The CIFAR-10 and CIFAR-100 datasest~\cite{cifar10} is a collection of 60,000 32$\times$32 color images with 50,000 training samples and 10,000 testing samples. All images are evenly distributed among 10/100 different classes, respectively. We split the training dataset over 100 clients for IID cases. For non-IID cases, we use Dirichlet distribution to simulate the non-IID settings on CIFAR-10 and CIFAR-100 datasets, which is controlled by a non-IID degree hyperparameter $\alpha$. The default value of $\alpha$ is set to 0.5 in our work.

For MNIST, FMNIST, and FEMNIST datasets, given their identical image format and size, we use the same neural network architecture in \cite{hu2023federated}. Specifically, we use a CNN model composed of two convolutional layers, each followed by max-pooling and ReLU activation functions. Two linear layers are utilized to map features to classes. For CIFAR-10/100 datasets, we use ResNet-18~\cite{resnet}. For the Shakespeare dataset, we implement a Recurrent Neural Network (RNN) model following \cite{reddi2020adaptive}. The RNN model takes a sequence of characters as input and then uses an embedding layer to convert each character into an 8-dimensional feature representation. Subsequently, two Long Short-Term Memory (LSTM) layers process these embedded characters, and a final linear layer with the softmax activation is applied.

For all datasets except CIFAR-10/100, the server randomly selects $h=100$ clients per round to perform local computations. While for CIFAR-10/100, we set $h=25$. We use SGD with momentum as the local solver, with the decay ratio and momentum parameters set to 0.99 and 0.9, respectively, for all datasets except for Shakespeare, where it is set to 0.999 and 0.5, respectively. The learning rate is set as $\eta=0.1$ for all datasets except for Shakespeare, where it is set to $\eta=1.0$. By default, the filtering radius is set as $\lambda_m = \lambda_d = 1.0$ for CIFAR-10/100. While for other datasets, we set $\lambda_m$ to 2.0. We define the sparsification level (SL) to be $1 - k/d$. A higher SL implies more parameters are zeroed out. In our experiments, SL is set as 0.3 for all datasets by default.
We run each experiment with three random seeds and report the average of the best testing accuracies achieved in each individual training. 
The experiments are conducted using PyTorch and executed on NVIDIA RTX A6000 GPUs.

\subsection{Evaluated attack methods}\label{attackmethods}
We consider eight attack methods including three naive attack methods, and five SOTA attack methods to comprehensively evaluate our method. 
\begin{itemize}
    \item \textit{Random attack.} The malicious clients send randomized updates that follow a Gaussian distribution $N(\mu, \sigma^2\textbf{I}_d)$. We set $\mu=(0,\dots,0) \in \mathbb{R}^d$ and $\sigma=0.5$.
    \item \textit{Noise attack.} The malicious clients perturb benign updates by adding Gaussian noise used in random attacks.
    \item \textit{Sign-flip attack.} The malicious clients manipulate their model updates by flipping the sign coordinately.
    \item \textit{Min-Max/Min-Sum attack~\cite{dnc}.} The malicious model updates are crafted in two steps. In the first step, the attacker generates a malicious update by perturbing the average of all benign updates. Then, for Min-Max attack, the attacker optimizes the malicious update so that its maximum Euclidean distance with any benign update is upper-bounded by the maximum distance between any two benign updates, i.e., $\max_{i,j\in\mathcal{H}}\|\Delta_i-\Delta_j\|_2$. For Min-Sum attack, the malicious update is optimized to ensure that the sum of its distances with each benign update is upper-bounded by the maximum total distance of a benign update among other benign updates, i.e.,$\max_{i\in\mathcal{H}}\sum_{j\in\mathcal{H}}\|\Delta_i-\Delta_j\|_2$. We additionally test a stealthy version of Min-Sum attack, where 
    the distance of the malicious update from any benign update is bounded by the minimum (rather than maximum) total distance of benign updates. This stealthy version is tested on all the datasets except for MNIST. We follow \cite{dnc} to keep the updates of all malicious clients the same. 
    \item \textit{AGR-tailored Trimmed-mean attack~\cite{dnc}.} AGR-tailored Trimmed-mean (TailoredTrmean) attack is designed to attack the defense method Trmean proposed in \cite{trmean} by maximizing the Euclidean distance between the aggregated result of simple average and Trmean, respectively.
    \item \textit{Lie attack~\cite{lie}.} The malicious clients apply slight changes to their local benign updates, making it hard to be detected. Specifically, the malicious clients calculate the element-wise mean $\mu_j$ and standard error $\sigma_j$ of all updates and generate the element of malicious updates by $(\beta_i)_j = \mu_j-z \times \sigma_j$, where $j \in [d]$. The scaling factor $z$ is set to 0.5 for all experiments.
    \item \textit{ByzMean attack~\cite{signguard}.} The ByzMean attack makes the mean of updates arbitrary malicious updates. Specifically, it divides malicious clients into two groups, each with $m_1$ and $m_2$ clients, respectively. 
    Clients in the first group select any existing attack methods to generate their malicious updates, denoted as $\beta_{i, \forall i \in [m_1]}$. The clients in the second group generate their malicious updates to make the average of all updates exactly equal to the average of malicious updates in $[m_1]$, which can be expressed as $  \beta_{i, \forall i \in [m_2]}=\frac{(n-m_1)\times \beta_{i, \forall i \in [m_1]}-\sum_{i=f+1}^n\Delta_i}{m_2}$ assuming the first $f$ updates are malicious. 
    We follow the same setting in \cite{signguard}, where the Lie attack is selected as the base attack method for the first group, and the size of two groups are set as $m_1=\lfloor f/ 2 \rfloor$ and $m_2 = f-m_1$.
\end{itemize}

\subsection{Additional experimental results}
\label{more_results}
\begin{table*}[htbp]
  \centering
  \caption{The main results for MNIST, FEMNIST, and Shakespeare are presented. }
  \scalebox{0.78}{
    \begin{tabular}{cc|c|cccccccc|c}
    \toprule
    \multicolumn{1}{c}{\multirow{1}[4]{*}{\textbf{\makecell*[c]{Datasets \\ (Model) }}}} & \multirow{1}[4]{*}{\textbf{\makecell*[c]{Defense \\ Methods }}} &  \multicolumn{1}{c}{\multirow{1}[4]{*}{\textbf{\makecell*[c]{No \\ Attack }}}} & \multicolumn{3}{c}{\textbf{Naive Attacks}} & \multicolumn{5}{c}{\textbf{State-of-the-art Attacks}} & \multicolumn{1}{c}{\multirow{1}[4]{*}{\textbf{\makecell*[c]{Average \\ w/ Attacks }}}}\\
\cmidrule(r){4-6}  \cmidrule(r){7-11}        &       &       & Random & Noise & Sign-flip & TailoredTrmean & Min-Max & Min-Sum & Lie   & ByzMean & \\
    \midrule
    \multicolumn{1}{c}{\multirow{9}[2]{*}{\makecell*[c]{MNIST \\ (CNN) }}} & FedAvg  & \underline{97.85} & 19.28 & 32.25 & 96.89 & 11.01 & 94.16 & 94.22 & 96.86 & 10.24 & 56.36 \\
          & TrMean & 96.14 & 94.11 & 94.50 & 95.19 & 11.35 & 88.35 & 88.41 & 93.67 & 10.74 & 72.67 \\
          & GeoMed & 94.59 & 94.66 & 94.66 & 94.21 & 94.76 & 63.99 & 52.29 & 80.82 & 94.17 & 83.69 \\
          & Multi-Krum & 97.00 & 96.50 & 96.73 & 96.97 & 11.35 & 67.33 & 69.51 & 93.82 & 10.24 & 67.43 \\
          & Bulyan & 94.95 & 96.42 & 96.41 & 94.20 & 11.70 & 63.89 & 68.00 & 90.98 & 54.88 & 71.06 \\
          & DnC   & 97.69 & 96.57 & 96.58 & \underline{97.14} & 46.31 & 64.57 & 89.89 & 96.17 & 28.29 & 76.69 \\
          & SignGuard & 96.64 & \underline{97.70} & \underline{97.70} & 96.85 & \underline{97.78} & \underline{97.58} & \underline{97.46} & \textbf{97.58} & \underline{97.63} & \underline{97.54} \\
          & SparseFed & \textbf{97.86} & 19.18 & 31.69 & 96.85 & 11.01 & 94.11 & 94.22 & 96.86 & 10.24 & 56.27 \\
          & \cellcolor[rgb]{ .816,  .808,  .808}LASA (Ours)  & \cellcolor[rgb]{ .816,  .808,  .808}97.35 & \cellcolor[rgb]{ .816,  .808,  .808}\textbf{97.96} & \cellcolor[rgb]{ .816,  .808,  .808}\textbf{98.27} & \cellcolor[rgb]{ .816,  .808,  .808}\textbf{97.26} & \cellcolor[rgb]{ .816,  .808,  .808}\textbf{97.94} & \cellcolor[rgb]{ .816,  .808,  .808}\textbf{97.93} & \cellcolor[rgb]{ .816,  .808,  .808}\textbf{97.94} & \cellcolor[rgb]{ .816,  .808,  .808}\underline{97.54} & \cellcolor[rgb]{ .816,  .808,  .808}\textbf{97.94} & \cellcolor[rgb]{ .816,  .808,  .808}\textbf{97.85} \\
    \midrule
    \vspace{-14pt}
          &       &       &       &       &       &       &       &       &       &       &  \\
    \midrule
    \multicolumn{1}{c}{\multirow{9}[2]{*}{\makecell*[c]{FEMNIST \\ (CNN) }}} & FedAvg  & \textbf{84.27} & 42.60 & 48.15 & 81.30 & 5.58 & 58.76 & 81.68 & 81.11 & 1.28 & 50.43 \\
          & TrMean & 82.23 & 78.26 & 78.81 & 79.13 & 5.70 & 29.80 & 76.72 & 75.79 & 5.73 & 53.12 \\
          & GeoMed & 75.57 & 75.48 & 75.47 & 71.67 & 76.19 & 68.27 & 28.13 & 22.56 & 74.32 & 61.01 \\
          & Multi-Krum & 82.85 & 76.13 & 76.48 & 80.00 & 5.58 & 25.83 & 77.25 & 74.91 & 6.48 & 52.58 \\
          & Bulyan & 77.10 & 81.68 & 81.65 & 73.50 & 5.97 & 19.17 & 60.55 & 58.98 & 18.02 & 49.94 \\
          & DnC   & \underline{83.89} & 75.41 & 76.08 & 80.96 & 63.93 & 66.60 & 80.37 & 78.97 & 22.84 & 68.52 \\
          & SignGuard & 83.06 & \underline{83.75} & \underline{83.75} & \underline{79.43} & \underline{83.80} & \underline{83.80} & \underline{82.59} & \underline{82.58} & \underline{83.78} & \underline{82.68}\\
          & SparseFed & \textbf{84.27} & 42.24 & 48.07 & 81.29 & 5.58 & 60.06 & 81.71 & 81.05 & 1.28 & 50.41 \\
          & \cellcolor[rgb]{ .816,  .808,  .808}LASA (Ours)  & \cellcolor[rgb]{ .816,  .808,  .808}83.69 & \cellcolor[rgb]{ .816,  .808,  .808}\textbf{84.07} & \cellcolor[rgb]{ .816,  .808,  .808}\textbf{84.05} & \cellcolor[rgb]{ .816,  .808,  .808}\textbf{81.72} & \cellcolor[rgb]{ .816,  .808,  .808}\textbf{84.26} & \cellcolor[rgb]{ .816,  .808,  .808}\textbf{84.19} & \cellcolor[rgb]{ .816,  .808,  .808}\textbf{83.60} & \cellcolor[rgb]{ .816,  .808,  .808}\textbf{83.52} & \cellcolor[rgb]{ .816,  .808,  .808}\textbf{84.14} & \cellcolor[rgb]{ .816,  .808,  .808}\textbf{83.94} \\

    \midrule
    \vspace{-14pt}
          &       &       &       &       &       &       &       &       &       &       &  \\
    \midrule
    \multicolumn{1}{c}{\multirow{9}[2]{*}{\makecell*[c]{Shakespeare \\ (LSTM) }}} & FedAvg  & 63.74 & 45.00 & 47.28 & 60.43 & 39.01 & 59.17 & 63.35 & \underline{62.79} & 24.24 & 50.41 \\
          & TrMean & 63.15 & 59.09 & 59.43 & 59.83 & 42.23 & 57.54 & 62.60 & 61.86 & 37.38 & 54.75 \\
          & GeoMed & 57.63 & 57.67 & 57.67 & 52.55 & 57.89 & 57.72 & 57.89 & 56.24 & 56.28 & 57.24 \\
          & Multi-Krum & 62.26 & 61.55 & 61.73 & 59.11 & 35.11 & 54.30 & 62.09 & 58.34 & 23.16 & 52.92 \\
          & Bulyan & 60.89 & 62.73 & 62.76 & 58.05 & 49.39 & 54.61 & 60.71 & 59.11 & 52.90 & 57.41 \\
          & DnC   & \underline{64.67} & 61.38 & 61.47 & \underline{60.80} & 59.32 & 61.10 & \textbf{64.70} & 62.30 & 56.18 & 60.65 \\
          & SignGuard & 63.65 & \underline{65.26} & \underline{65.26} & 59.84 & \underline{64.76} & \underline{64.76} & 60.83 & 62.35 & \underline{64.76} & \underline{61.97} \\
          & SparseFed & 63.72 & 44.49 & 47.24 & 60.40 & 39.24 & 59.84 & 63.31 & 62.77 & 24.27 & 50.69 \\
          & \cellcolor[rgb]{ .816,  .808,  .808}LASA (Ours)  & \cellcolor[rgb]{ .816,  .808,  .808}\textbf{65.08} & \cellcolor[rgb]{ .816,  .808,  .808}\textbf{66.25} & \cellcolor[rgb]{ .816,  .808,  .808}\textbf{66.24} & \cellcolor[rgb]{ .816,  .808,  .808}\textbf{62.56} & \cellcolor[rgb]{ .816,  .808,  .808}\textbf{66.32} & \cellcolor[rgb]{ .816,  .808,  .808}\textbf{65.63} & \cellcolor[rgb]{ .816,  .808,  .808}\underline{64.02} & \cellcolor[rgb]{ .816,  .808,  .808}\textbf{64.25} & \cellcolor[rgb]{ .816,  .808,  .808}\textbf{65.99} & \cellcolor[rgb]{ .816,  .808,  .808}\textbf{65.16} \\
    \bottomrule
    \end{tabular}%
    }
  \label{tab:main_results_extra}%
\end{table*}%

In this section, we set the attack ratio to 25\%, and for FEMNIST and Shakeperare datasets, we set $\lambda_d$ to 1.5. As shown in Table~\ref{tab:main_results_extra}, LASA demonstrates its robustness against the naive and SOTA attack methods in IID settings, whereas almost all other defense methods are vulnerable to at least one attack method. Under no attack, LASA achieves a test accuracy comparable to FedAvg on MNIST dataset. This demonstrates the effectiveness of LASA in maintaining accuracy, not just in adversarial environments, but also in benign environments. 

For MNIST dataset, LASA achieves the best performance against naive attacks with the highest accuracy of 97.96\% for Random attack, 98.27\% for Noise attack, and 97.26\% for Sign-Flip attack, outperforming all other defense methods. In contrast, SignGuard, DnC and LASA can effectively defend against TailoredTrmean and ByzMean attacks. Under TailoredTrmean attack, LASA achieves the highest accuracy of 97.94\%, which is +0.17\% and +51.63\% higher than SignGuard and DnC, respectively; under ByzMean attacks, LASA achieves the highest accuracy of 97.94\%, which is 0.31\% and +69.66\% higher than SignGuard and DnC, respectively.

\textbf{Main results in non-IID settings.} Compared to FedAvg under no attack, we can see that LASA can maintain the accuracy of FL in the benign environment with only a -0.57\% accuracy drop on FEMNIST dataset and even a +1.34\% accuracy increase on Shakespeare dataset. We also observe that the performance of classic robust aggregation rules, including Trmean, GeoMed, Multi-Krum, and Bulyan, is poor on non-IID datasets. For example, Trmean and Multi-Krum completely failed against the ByzMean attack on FEMNIST dataset, yielding an accuracy of 5.73\% and 6.48\%, respectively. As we discussed in the related works, in non-IID settings, the divergence between benign model updates will increase, making these classic methods hard to filter out malicious model updates. For FEMNIST dataset, LASA outperforms all other defense methods. It achieves an accuracy of 84.26\% at best under TailoredTrmean attack, which is identical to that of Mean under no attack. In addition, LASA outperforms SignGuard more significantly in non-IID settings, compared to their performance in IID settings. Specifically on Shakespeare dataset, the performance of SignGuard is not stable. For example, under Sign-Flip attack, the accuracy of SignGuard drops to 59.84\%, while LASA achieves the highest accuracy of 62.56\% (+2.72\%). Under Min-Sum attack, SignGuard's accuracy drops to 60.83\%, while LASA achieves an accuracy of 64.017\% (+3.19\%), which is comparable to the best accuracy achieved by DnC. 

In a nutshell, the performance of LASA is not only manifested in attack scenarios but also in the absence of any attacks, which aligns with the design principles of LASA. Moreover, LASA shows robustness to both IID and more challenging non-IID cases. By adeptly integrating pre-aggregation sparsification and layer-wise adaptive aggregation, LASA effectively mitigates the impact of updates that diverge from others. The robustness of LASA, illustrated by the above-mentioned results, emphasizes its potential as a robust defense method in securing federated learning environments against a wide collection of attacks, ultimately enhancing the reliability of federated learning systems.

\begin{figure*}[ht]
    \centering
        \includegraphics[width=\linewidth]{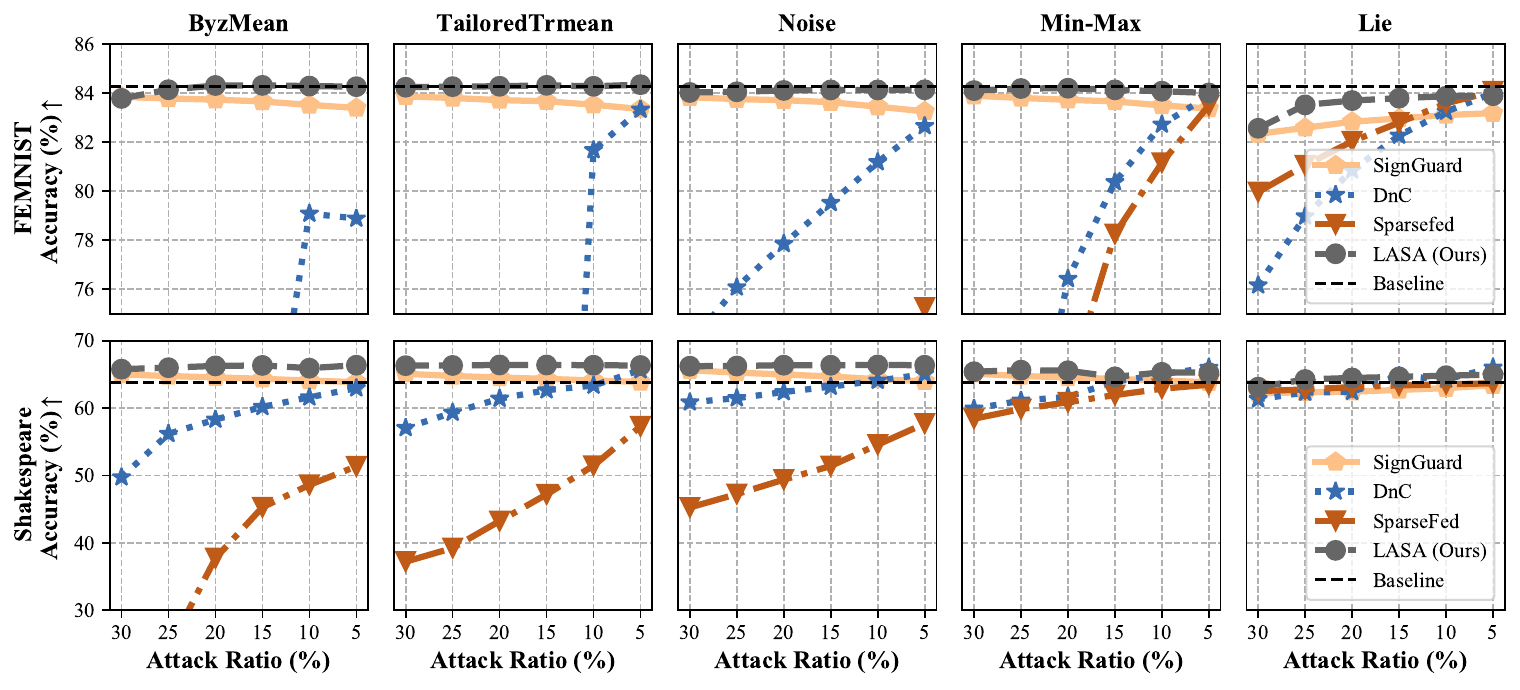}
    \caption{Testing Accuracy of LASA, SignGuard, DnC and SparseFed under Various Attack Ratios in non-IID Settings.}
    \label{fig:noniid}
\end{figure*}
\subsection{More results under various attack ratios}
\label{more_attack_ratios_results}

We evaluate the performance of three SOTA defense methods including DnC, SignGuard and SparseFed, and our method LASA under different attack ratios on non-IID datasets and report the results in Fig.~\ref{fig:noniid}. Specifically, we conduct experiments under one naive attack and four SOTA attacks with the attack ratio varying from 5\% to 30\%. In Fig.~\ref{fig:noniid}, the \textit{Baseline} represents the non-robust method Mean under no attack. In general, DnC and SparseFed's accuracies increase as the attack ratio decreases, but they suffer from significant accuracy degradation when the attack ratio is high, especially under Byzmean and TailoredTrmean attacks. For instance, on FEMNIST dataset, even when the attack ratio is as low as 5\%, SparseFed does not improve the robustness, achieving an accuracy of 7.44\% under the ByzMean attack. Similarly, DnC struggles to defend against ByzMean attack effectively until the attack ratio is reduced to 10\%, achieving a relatively low accuracy of 79.09\%. SignGuard outperforms DnC and SparseFed significantly. However, under Byzmean, TailoredTrmean, Noise, and Min-Max attacks, the accuracy of SignGuard decreases as the attack ratio decreases. Compared to SignGuard, our method LASA achieves a better and more stable performance. As the attack ratio increases, LASA only has a minor decrease in accuracy. 

\subsection{Impact of sparsification level}
\label{sec:discuss_of_k}

\begin{table}[t]
  \centering
  \caption{Performance of LASA with Different Sparsification Levels.}
  \scalebox{0.68}{
    \begin{tabular}{cc|ccccccc}
    \toprule
    \multirow{2}[4]{*}{\textbf{Att.}} & \multirow{2}[4]{*}{\textbf{Data.}} & \multicolumn{7}{c}{\textbf{Sparsification Level}} \\
\cmidrule{3-9}          &       & 0.1   & 0.3   & 0.5   & 0.7   & 0.9   & 0.95  & 0.99 \\
    \midrule
    \multirow{4}[2]{*}{\rotatebox{90}{ByzMean}} & M & 97.833 & 97.943 & 97.821 & 97.543 & \textbf{98.053} & 97.880 & 97.490 \\
          & FM & 87.820 & 87.647 & \textbf{87.943} & 87.867 & 87.803 & 87.740 & 86.437 \\
          & FEM & \textbf{84.143} & 84.138 & 84.137 & 84.118 & 83.834 & 83.489 & 81.120 \\
          & Sha   & \textbf{66.024} & 65.990 & 65.842 & 65.409 & 64.355 & 63.055 & 60.463 \\
    \midrule
        \vspace{-11pt}
          &       &       &       &       &       &       &       &  \\
    \midrule
    \multirow{4}[2]{*}{\rotatebox{90}{Min-Max}} & M & 97.310 & \textbf{97.930} & 97.493 & 97.307 & 97.557 & 96.950 & 97.593 \\
          & FM & 87.917 & 87.907 & 87.920 & \textbf{87.967} & 87.330 & 87.707 & 86.353 \\
          & FEM & 84.184 & \textbf{84.264} &84.203 & 84.162 & 83.772 & 83.361 & 81.038 \\
          & Sha   & 64.723 & \textbf{66.324} & 65.472 & 65.032 & 63.654 & 62.527 & 60.090 \\
    \midrule
    \vspace{-11pt}
          &       &       &       &       &       &       &       &  \\
    \midrule
    \multirow{4}[2]{*}{\rotatebox{90}{Noise}} & M & 98.223 & 98.270 & \textbf{98.320} & 97.643 & 98.050 & 97.923 & 97.800 \\
          & FM & 87.877 & 87.870 & 87.893 & \textbf{87.933} & 87.623 & 87.897 & 86.423 \\
          & FEM & \textbf{84.061} & 84.053 & 84.023 & 84.018 & 83.645 & 83.269 & 80.537 \\
          & Sha   & \textbf{66.255} & 66.244 &   66.102   &   65.754    & 64.506 &   63.546    & 60.686 \\
    \bottomrule
    \end{tabular}%
    }
  \label{tab:spar_level}%
\end{table}%

As we stated in Section 4.1, the optimal sparsification parameter $k$ should balance the tradeoff between sparsification error and robustness improvement. Here, we empirically study the impact of different $k$ on learning performance. Recall that the SL is defined as $1 - k/d$, hence, a smaller $k$ implies a higher SL and a heavier sparsification. We report the performance of LASA under Noise, Min-Max, and ByzMean attacks with SLs varying from 0.1 to 0.99 in Table~\ref{tab:spar_level}, where M, FM, FEM, and Sha represent MNIST, FMNIST, FEMNIST, and Shakespeare datasets, respectively. The results demonstrate that there exists an optimal SL that maximizes robustness and a very high SL may lead to a significant accuracy drop. For example, as SL increases, the accuracy of LASA on FMNIST dataset increases to 87.94\% and then decreases to 86.44\% under ByzMean attack. This occurs because the sparsification error overwhelms the robustness improvement when SL is too large. We also observe that the sensitivity of LASA on SL depends on both the dataset and the attack method.

\subsection{Impact of filtering radius}
\begin{table}[t]
  \centering
  \caption{Performance of LASA with Different Filtering Radius}
  \scalebox{0.72}{
    \begin{tabular}{cc|cccccc}
    \toprule
    \multicolumn{2}{c}{\textbf{Con.}} & \multicolumn{2}{c}{\textbf{MNIST}} & \multicolumn{2}{c}{\textbf{FMNIST}} & \multicolumn{2}{c}{\textbf{FEMNIST}} \\
    \cmidrule(r){1-2}  \cmidrule(r){3-4} \cmidrule(r){5-6} \cmidrule(r){7-8}
    $\lambda_d$ & $\lambda_m$ & Noise & ByzMean & Noise & ByzMean & Noise & ByzMean \\
    \midrule
    1.0   & 1.0 & 97.963 & 97.803 &87.950 & \textbf{87.887}& 83.922 & 84.158 \\
    1.0   &  1.5     & 97.883 & 97.843 & \textbf{88.023} & 87.720 & 83.946 & \textbf{84.209} \\
    1.0   &   2.0    & \textbf{98.270} & 97.943 & 87.870 & 87.647 & 84.007 & 84.119 \\
    1.0   &   4.0    & 91.743 &  97.840    & 77.400 & 77.930 & 69.408 & 84.048 \\
    1.5 & 2.0   & 97.927 & \textbf{98.023} & 87.937 & 87.640 & 84.053 & 84.138 \\2.0
          & 2.0   & 97.593 & 97.487 & 87.950 & 84.000&  84.136    & 77.399 \\3.0
          & 2.0   & 97.883 & 66.897 & 87.917 & 67.250 &   \textbf{84.225}   & 28.300 \\
    \bottomrule
    \end{tabular}%
    }
  \label{tab:filter}%
\end{table}%
\label{disscuss_of_lambda}
In this subsection, we study the performance of LASA with different filtering radius $\lambda_m$ and $\lambda_d$. A smaller $\lambda_m$ or $\lambda_d$ indicates more stringent filtering and results in a smaller benign set for aggregation. As shown in Table~\ref{tab:filter}, there exist optimal $\lambda_m$ and $\lambda_d$ that balance the filtering intensity and maximize the model accuracy. 
We also observe that the effectiveness of Noise attack is marginally affected by $\lambda_d$, as random noise perturbation does not change the sign purity in expectation. For all datasets, the optimal $\lambda_d$ under Noise attack is 1.0 (note that for FEMNIST, the best accuracy when $\lambda_d=3.0$ is comparable to the accuracy when $\lambda_d=1.0$). However, as Noise attack adds Gaussian noise to the model updates to increase their magnitude (in $L_2$ norm), the effectiveness of Noise attack is sensitive to the values of $\lambda_m$. For different datasets, the optimal $\lambda_m$ are different. For the advanced ByzMean attack, its effectiveness is marginally affected by $\lambda_m$, as the accuracy of LASA does not change much when $\lambda_m$ increases from 1.0 to 2.0. This demonstrates that the magnitudes of malicious updates generated by ByzMean attack are close to that of benign models. In order to make the attack effective, ByzMean attack mainly focuses on manipulating the model direction, making it sensitive to the direction filtering radius $\lambda_d$: the accuracy of LASA vibrates a lot as $\lambda_d$ increases. Additionally, both $\lambda_m$ and $\lambda_d$ should not be too large to compromise the effectiveness of the filtering.

\subsection{Computational cost of LASA}
We evaluate the computational cost of LASA in comparison to other methods. LASA incorporates pre-aggregation sparsification, leading to a complexity of $O(d\log d)$ due to the use of sorting algorithms like \textit{merge sort} in the parameter space of local updates. Consequently, the worst-case computational expense for LASA is $O(nd\log d)$. Despite this, LASA's computational burden is on par with other methods such as Krum and Multi-Krum, which have a complexity of $O(dn^2)$, and Trmean with $O(dn\log n)$.
\clearpage
\onecolumn
\subsection{Proof preliminaries}

\subsubsection{Useful Inequalities}

\begin{lemma}Given any two vectors $a, b \in \mathbb{R}^d$,
\label{lemma_for_product}
\begin{align*}
    2\left < a ,b \right > \leq \alpha \left \| a \right \|^2 + \frac{1}{\alpha} \left \| b \right \|^2, \forall \; \alpha > 0.
\end{align*} 
\end{lemma}

\begin{lemma}Given any two vectors $a, b \in \mathbb{R}^d$,
\label{lemma_for_break_norm}
\begin{align*}
    \left\| a + b \right \|^2 \leq (1+\delta)\left\| a \right \|^2 + (1 + \delta ^{-1})\left\| b \right \|^2, \forall \; \delta > 0.
\end{align*} 
\end{lemma}
\begin{lemma} \label{lemma_for_extract_sum}Given arbitrary set of n vectors $\{a_i\}^n_{i=1}$, $a_i \in \mathbb{R}^d$,
\begin{align*}
    \left\|  \sum^n_{i=1} a_i \right \|^2 \leq n\sum^n_{i=1}\left \| a_i\right \|^2.
\end{align*}
\end{lemma}
\begin{lemma}\label{local_divergence}
    If the learning rate $\eta\leq 1/2\tau$, under Assumption 2 and 3, the local divergence of benign model updates are bounded as follows:
    \begin{equation}
        \frac{1}{|\mathcal{B}|} \sum_{i \in \mathcal{B}} \mathbb{E}\left\|  {\Delta}_i - \bar{\Delta}_{\mathcal{B}} \right\|^2 \leq 2\bar{\nu} + \bar{\zeta}
    \end{equation}
\end{lemma}
\begin{proof}
Given that $\Delta_i = \eta \sum^{\tau-1}_{s=0} g^{s}_i$ where $\eta$ is the learning rate and $g^{s}_i$ is the local stochastic gradient over the mini-batch $s$. We have
\begin{align}
    \frac{1}{|\mathcal{B}|}\sum_{i \in \mathcal{B}} \mathbb{E}\left \|  \Delta_i - \bar{\Delta}_\mathcal{B} \right \|^2 
    &= \frac{1}{|\mathcal{B}|}\sum_{i \in \mathcal{B}} \mathbb{E}\left \|  \eta \sum^{\tau - 1}_{s=0} g^{s}_i - \frac{1}{|\mathcal{B}|}\sum_{i \in \mathcal{B}}\eta\sum^{\tau - 1}_{s=0} g^{s}_i \right \|^2 \\
    & =\frac{\eta^2}{|\mathcal{B}|}\sum_{i \in \mathcal{B}} \mathbb{E} \left \|  \sum^{\tau - 1}_{s=0} g^{s}_i - \frac{1}{|\mathcal{B}|}\sum_{i \in \mathcal{B}} \sum^{\tau - 1}_{s=0} g^{s}_i \right \|^2 \nonumber \\
    & \leq \frac{\tau \eta^2}{|\mathcal{B}|}\sum_{i \in \mathcal{B}}\sum^{\tau - 1}_{s=0} \mathbb{E}\left \|  g^{s}_i - \frac{1}{|\mathcal{B}|}\sum_{i \in \mathcal{B}}  g^{s}_i \right \|^2 \nonumber \\
    & = \frac{\tau \eta^2}{|\mathcal{B}|}\sum_{i \in \mathcal{B}}\sum^{\tau - 1}_{s=0} \mathbb{E}\left \|  \left( g^{s}_i - \nabla \mathcal{L}_i(\theta^{s}_i) \right)  + \left( \nabla \mathcal{L}_\mathcal{B}(\theta^{s}_i) - \frac{1}{|\mathcal{B}|}\sum_{i \in \mathcal{B}}  g^{s}_i \right) + \left( \nabla \mathcal{L}_i(\theta^{s}_i) -\nabla \mathcal{L}_\mathcal{B}(\theta^{s}_i) \right )  \right \|^2 \nonumber \\
    & \leq \frac{3 \tau \eta^2}{|\mathcal{B}|}\sum_{i \in \mathcal{B}}\sum^{\tau - 1}_{s=0} \underbrace{\mathbb{E} \left \|  g^{s}_i - \nabla \mathcal{L}_i(\theta^{s}_i) )\right \|^2}_{T_1}  + \frac{3 \tau \eta^2}{|\mathcal{B}|}\sum_{i \in \mathcal{B}}\sum^{\tau - 1}_{s=0}\underbrace{\mathbb{E} \left \| \nabla \mathcal{L}_\mathcal{B}(\theta^{s}_i) - \frac{1}{|\mathcal{B}|}\sum_{i \in \mathcal{B}}  g^{s}_i \right \|^2}_{T_2} \nonumber \\
    & \quad + \underbrace{\frac{3 \tau \eta^2}{|\mathcal{B}|}\sum_{i \in \mathcal{B}}\sum^{\tau - 1}_{s=0}\mathbb{E}\left \|  \nabla \mathcal{L}_i(\theta^{s}_i) -\nabla \mathcal{L}_\mathcal{B}(\theta^{s}_i)   \right \|^2}_{T_3}, \label{proof_our_kappa_t1}
\end{align}
where the first inequality follows Lemma~\ref{lemma_for_extract_sum}, and the last second follows Lemma~\ref{lemma_for_break_norm}. For $T_1$, with Assumption 2, we have
\begin{align}
    T_1 \leq \bar{\nu}. \label{proof_our_kappa_t2}
\end{align}
For $T_2$, we have
\begin{align}
    T_2 = \mathbb{E}\left \| \nabla \mathcal{L}_\mathcal{B}(\theta_i^{ s}) - \frac{1}{|\mathcal{B}|}\sum_{i \in \mathcal{B}}  g^{ s}_i \right \|^2 = \mathbb{E}\left \| \frac{1}{|\mathcal{B}|}\sum_{i \in \mathcal{B}} \left(\nabla \mathcal{L}_i(\theta_i^{ s}) -  g^{ s}_i\right) \right \|^2  
    \leq \frac{1}{|\mathcal{B}|}\sum_{i \in \mathcal{B}} \mathbb{E}\left \| \nabla \mathcal{L}_i(\theta_i^{ s}) -  g^{s}_i \right \|^2 
    \leq  \bar{\nu},\label{proof_our_kappa_t3}
\end{align}
where the first inequality follows Lemma~\ref{lemma_for_extract_sum}, and the last inequality follow Assumption 2. For $T_3$, we have
\begin{align}
    T_3 = \frac{3 \tau \eta^2}{|\mathcal{B}|}\sum_{i \in \mathcal{B}}\sum^{\tau - 1}_{s=0}\mathbb{E}\left \|  \nabla \mathcal{L}_i(\theta^{ s}_i) -\nabla \mathcal{L}_\mathcal{B}(\theta_i^{ s})   \right \|^2
    \leq 3 \tau \eta^2\sum^{\tau - 1}_{s=0} \bar{\zeta} = 3 \tau^2 \eta^2 \bar{\zeta}\label{proof_our_kappa_t4}
\end{align}
by Assumption 3. 

\noindent Plugging~\ref{proof_our_kappa_t2}, \ref{proof_our_kappa_t3}, and~\ref{proof_our_kappa_t4} back to~\ref{proof_our_kappa_t1}, with $\eta\leq 1/2\tau$, we have
\begin{align*}
    \frac{1}{|\mathcal{B}|}\sum_{i \in \mathcal{B}} \mathbb{E}\left \|  \Delta_i - \bar{\Delta}_\mathcal{B} \right \|^2 \leq 2\bar{\nu} + \bar{\zeta}.
\end{align*}
This concludes the proof.
\end{proof}

\clearpage
\subsubsection{Proof of Lemma 1}
\label{proofofpro1}

\begin{proof}
Recall that LASA denoted by $F(\cdot) \colon \mathbb{R}^{d \times n} \rightarrow \mathbb{R}^d$ is a layer-wise aggregation rule, i.e., there exist $L$ real-valued functions $F_1, \dots, F_L: \mathbb{R}^{d\times n} \rightarrow \mathbb{R}^{d}$ such that for all $\Delta_1, \dots, \Delta_n \in \mathbb{R}^d$, $[F(\Delta_1, \dots, \Delta_n)]_l = F_l(\Delta^l_1, \dots, \Delta_n^l)$. As LASA utilizes layer-wise aggregation, we have 
\begin{align}
    F_l(\Delta_1, \dots, \Delta_n) = \frac{1}{|\mathcal{S}^l|} \sum_{i \in \mathcal{S}^l} \hat{\Delta}^l_i, \nonumber 
\end{align}
where $ \hat{\Delta}^l_i $ be the $l$-th layer of the Top-$k$ sparsified model $\hat{\Delta}_i$ and $\mathcal{S}^l$ is the indices set of benign updates in $l$-th layer shown in Algorithm 1. We denote the indices set of Top-$k$ parameters of a model/layer by $\mathcal{K}$ and the set of remaining parameters by $\mathcal{K}^-$. Let $[\Delta]_{\mathcal{K}}$ represent a sparsified model with only parameters in $\mathcal{K}$ (the rest are zero), then we have
\begin{align*}
\mathbb{E} \|F(\Delta_1, \dots, \Delta_n) - \bar{\Delta}_\mathcal{B}\|^2 & = 
\mathbb{E}\sum_{l=1}^L\left\|F_l(\Delta_1, \dots, \Delta_n) - \bar{\Delta}^l_\mathcal{B}\right\|^2\\ 
& = \mathbb{E} \sum_{l=1}^L\left \|\frac{1}{|\mathcal{S}^l|}\sum_{i \in \mathcal{S}^l} \hat{\Delta}^l_i - \bar{\Delta}^l_{\mathcal{B}} \right\|^2 \\
& =\mathbb{E} \sum_{l=1}^L\frac{1}{|\mathcal{S}^l|^2} \left \|\sum_{i \in \mathcal{S}^l} \hat{\Delta}^l_i - \bar{\Delta}^l_{\mathcal{B}} \right\|^2 \\
& =\mathbb{E} \sum_{l=1}^L\frac{1}{|\mathcal{S}^l|^2} \left\| \sum_{i \in \mathcal{S}^l}\left[ \hat{\Delta}^l_i - \bar{\Delta}^l_{\mathcal{B}} \right]_{\mathcal{K}_i^l}  + \sum_{i \in \mathcal{S}^l}\left[ - \bar{\Delta}^l_{\mathcal{B}} \right]_{\mathcal{K}_i^{l-}}\right\|^2  \\
& =\mathbb{E} \sum_{l=1}^L\frac{1}{|\mathcal{S}^l|^2} \left\| \sum_{i \in \mathcal{S}^l}\left[ {\Delta}^l_i - \bar{\Delta}^l_{\mathcal{B}} \right]_{\mathcal{K}_i^l}  + \sum_{i \in \mathcal{S}^l}\left[ - \bar{\Delta}^l_{\mathcal{B}} \right]_{\mathcal{K}_i^{l-}}\right\|^2  \\
& =\mathbb{E} \sum_{l=1}^L \frac{1}{|\mathcal{S}^l|^2} \left(\left\| \sum_{i \in \mathcal{S}^l}\left[ {\Delta}^l_i - \bar{\Delta}^l_{\mathcal{B}} \right]_{\mathcal{K}_i^l} \right\|^2+ \left\|\sum_{i \in \mathcal{S}^l}\left[  \bar{\Delta}^l_{\mathcal{B}} \right]_{\mathcal{K}_i^{l-}}\right\|^2 \right).
\end{align*}

\noindent Let $c_i^l := \left\| \left[ {\Delta}^l_i - \bar{\Delta}^l_{\mathcal{B}} \right]_{\mathcal{K}_i^l} \right\|^2 / \left\|  {\Delta}_i - \bar{\Delta}_{\mathcal{B}} \right\|^2$, $ b_{\mathcal{B}}^l:= \left\|\left[  \bar{\Delta}^l_{\mathcal{B}} \right]_{\mathcal{K}_i^{l-}}\right\|^2/\left\| \bar{\Delta}_{\mathcal{B}} \right\|^2$, $C_{\mathcal{B}}^2:= \left\| \bar{\Delta}_{\mathcal{B}} \right\|^2$, $b_{\mathcal{B}}:= \sum_{l=1}^L b_{\mathcal{B}}^l $, and $c_i := \sum_{l=1}^Lc_i^l $, we have
\begin{align}
    \mathbb{E}\|F(\Delta_1, \dots, \Delta_n) - \bar{\Delta}_\mathcal{B}\|^2
    & \leq \mathbb{E}\sum_{l=1}^L \frac{1}{|\mathcal{S}^l|} \sum_{i \in \mathcal{S}^l}\left(\left\| \left[ {\Delta}^l_i - \bar{\Delta}^l_{\mathcal{B}} \right]_{\mathcal{K}_i^l} \right\|^2+ \left\|\left[  \bar{\Delta}^l_{\mathcal{B}} \right]_{\mathcal{K}_i^{l-}}\right\|^2 \right) 
 \nonumber \\
    &  = \mathbb{E}\sum_{l=1}^L \frac{1}{|\mathcal{S}^l|} \sum_{i \in \mathcal{S}^l} \left(c_i^l\left\|  {\Delta}_i - \bar{\Delta}_{\mathcal{B}} \right\|^2+ b_{\mathcal{B}}^l\left\| \bar{\Delta}_{\mathcal{B}} \right\|^2 \right)\nonumber \\
    & = \mathbb{E}\sum_{l=1}^L \frac{1}{|\mathcal{S}^l|} \sum_{i \in \mathcal{S}^l} \left(c_i^l\left\|  {\Delta}_i - \bar{\Delta}_{\mathcal{B}} \right\|^2+ b_{\mathcal{B}}^lC_{\mathcal{B}}^2 \right)\nonumber,\\
    & = \mathbb{E}\sum_{l=1}^L \frac{1}{|\mathcal{S}^l|} \sum_{i \in \mathcal{S}^l} c_i^l\left\|  {\Delta}_i - \bar{\Delta}_{\mathcal{B}} \right\|^2+ C_{\mathcal{B}}^2\sum_{l=1}^L b_{\mathcal{B}}^l \nonumber\\
    & = \underbrace{\mathbb{E}\frac{1}{|\mathcal{S}^l|} \sum_{i \in \mathcal{S}^l} c_i\left\|  {\Delta}_i - \bar{\Delta}_{\mathcal{B}} \right\|^2 }_{T_1}+ C_{\mathcal{B}}^2b_{\mathcal{B}},
    \label{T_1_start}
\end{align}
where the first inequality follows Lemma~\ref{lemma_for_extract_sum}. Note that $c_i = \left\| \left[ {\Delta}_i - \bar{\Delta}_{\mathcal{B}} \right]_{\mathcal{K}_i} \right\|^2 / \left\|  {\Delta}_i - \bar{\Delta}_{\mathcal{B}} \right\|^2$ and $ b_{\mathcal{B}}= \left\|\left[  \bar{\Delta}_{\mathcal{B}} \right]_{\mathcal{K}_i^{-}}\right\|^2/\left\| \bar{\Delta}_{\mathcal{B}} \right\|^2$.

\noindent Now we treat $T_1$. If $\mathcal{S}^l \subseteq \mathcal{B}$, we have 
\begin{align}
    T_1 = \mathbb{E}\left[\frac{1}{|\mathcal{S}^l|} \sum_{i \in \mathcal{S}^l} c_i \left\|  {\Delta}_i - \bar{\Delta}_{\mathcal{B}} \right\|^2 \right]
    &\leq  
    \mathbb{E}\left[\frac{1}{|\mathcal{S}^l|} \sum_{i \in \mathcal{B}} c_i\left\|  {\Delta}_i - \bar{\Delta}_{\mathcal{B}} \right\|^2\right].\label{kappa_in_1} 
\end{align}

\noindent If $\mathcal{S}^l \nsubseteq \mathcal{B}$, let $\mathcal{P} = \mathcal{S}^l \cap \mathcal{B}$, and $\mathcal{R} = \mathcal{S}^l \backslash \mathcal{B}$, let $C_{\mathcal{M},i}^2: = \left\|  {\Delta}_i\right\|^2, \forall i \in [N]\backslash \mathcal{B}$, then we have
\begin{align}
    T_1 = \mathbb{E}\left[\frac{1}{|\mathcal{S}^l|} \sum_{i \in \mathcal{S}^l} c_i\left\|  {\Delta}_i - \bar{\Delta}_{\mathcal{B}} \right\|^2 \right]
    &=  
    \mathbb{E}\left[\frac{1}{|\mathcal{S}^l|} \left ( \sum_{i \in \mathcal{P}} c_i\left\|  {\Delta}_i - \bar{\Delta}_{\mathcal{B}} \right\|^2 + \sum_{i \in \mathcal{R}} c_i\left\|  {\Delta}_i - \bar{\Delta}_{\mathcal{B}} \right\|^2 \right ) \right]\nonumber \\
    &\leq
    \mathbb{E}\left[\frac{1}{|\mathcal{S}^l|}  \sum_{i \in \mathcal{B}} c_i\left\|  {\Delta}_i - \bar{\Delta}_{\mathcal{B}} \right\|^2\right] + \mathbb{E}\left[\frac{1}{|\mathcal{S}^l|} \sum_{i \in \mathcal{R}} c_i\left\|  {\Delta}_i - \bar{\Delta}_{\mathcal{B}} \right\|^2\right] \nonumber \\
    & \leq
    \mathbb{E}\left[\frac{1}{|\mathcal{S}^l|}  \sum_{i \in \mathcal{B}} c_i\left\|  {\Delta}_i - \bar{\Delta}_{\mathcal{B}} \right\|^2\right] + \mathbb{E}\left[\frac{2}{|\mathcal{S}^l|} \sum_{i \in \mathcal{R}}c_i \left ( \left\|  {\Delta}_i\right\|^2 +  \left \| \bar{\Delta}_{\mathcal{B}} \right\|^2 \right ) \right]\nonumber \\
    & =
    \mathbb{E}\left[\frac{1}{|\mathcal{S}^l|}  \sum_{i \in \mathcal{B}} c_i\left\|  {\Delta}_i - \bar{\Delta}_{\mathcal{B}} \right\|^2\right] + \mathbb{E}\left[\frac{2}{|\mathcal{S}^l|} \sum_{i \in \mathcal{R}}c_i \left ( C_{\mathcal{M},i}^2 + C_\mathcal{B}^2 \right )\right]
    \label{kappa_in_2},
\end{align}
where the second inequality follows Lemma~\ref{lemma_for_break_norm}.

\noindent Due to the use of MZ-score, models in $\mathcal{S}^l$ are centered around the median within a $\lambda_m$ (and $\lambda_d$) radius. If the radius parameter $\lambda_m$ or $\lambda_d$ equals to zero, only the median model (based on $l_2$-norm or PDP) will be selected for averaging. To maximize benign model inclusion in averaging, the radius parameters $\lambda_m$ and $\lambda_d$ are set sufficiently large to ensure $|\mathcal{S}^l|\geq n/2 -f$. More precisely, assume there exist two positive constants $\lambda^{+}_m$ and $\lambda^{+}_d$, and if the radius parameters $\lambda_m$ and $\lambda_d$ in Algorithm 1 satisfy $\lambda_m \geq \lambda^{+}_m, \lambda_d \geq \lambda^{+}_d$, we have $|\mathcal{S}^l|\geq n/2 -f$, $\forall l\in[L]$. Integrated with~\ref{kappa_in_1} and~\ref{kappa_in_2}, we have
\begin{align}
      T_1 
 & \leq \nonumber \begin{cases}
            \frac{2}{n-2f} \mathbb{E}\sum_{i \in \mathcal{B}} c_i\left\|  {\Delta}_i - \bar{\Delta}_{\mathcal{B}} \right\|^2 , & \text{if} \ \mathcal{S}^l \subseteq \mathcal{B}\\ \nonumber \\
            \frac{2}{n-2f} \mathbb{E} \sum_{i \in \mathcal{B}} c_i\left\|  {\Delta}_i - \bar{\Delta}_{\mathcal{B}} \right\|^2 + \frac{4}{n-2f}\mathbb{E} \sum_{i \in \mathcal{R}}c_i \left ( C_{\mathcal{M},i}^2 + C_\mathcal{B}^2 \right ), & \text{if} \ \mathcal{S}^l \nsubseteq \mathcal{B} 
        \end{cases}\nonumber\\ 
        & \leq \frac{2}{n-2f} \mathbb{E} \sum_{i \in \mathcal{B}} c_i\left\|  {\Delta}_i - \bar{\Delta}_{\mathcal{B}} \right\|^2 + \frac{4}{n-2f} \mathbb{E} \sum_{i \in \mathcal{R}}c_i \left ( C_{\mathcal{M},i}^2 + C_\mathcal{B}^2 \right ) \nonumber\\
        & \leq \frac{2 c_{max}}{n-2f} \mathbb{E} \sum_{i \in \mathcal{B}} \left\|  {\Delta}_i - \bar{\Delta}_{\mathcal{B}} \right\|^2 + \frac{4c_{max}}{n-2f}  \sum_{i \in \mathcal{R}} \left ( C_{\mathcal{M},i}^2 + C_\mathcal{B}^2 \right ) \nonumber\\
        & = \frac{2c_{max}|\mathcal{B}|}{n-2f} \frac{1}{|\mathcal{B}|} \sum_{i \in \mathcal{B}}\mathbb{E} \left\|  {\Delta}_i - \bar{\Delta}_{\mathcal{B}} \right\|^2 + \frac{4c_{max}}{n-2f} \sum_{i \in \mathcal{R}} \left ( C_{\mathcal{M},i}^2 + C_\mathcal{B}^2 \right )\nonumber \\
        & = \frac{2c_{max}(n-f)}{n-2f} \frac{1}{|\mathcal{B}|} \sum_{i \in \mathcal{B}} \mathbb{E}\left\|  {\Delta}_i - \bar{\Delta}_{\mathcal{B}} \right\|^2 + \frac{4c_{max}}{n-2f} \sum_{i \in \mathcal{R}} \left ( C_{\mathcal{M},i}^2 + C_\mathcal{B}^2 \right ) \nonumber\\
        & \leq \frac{2(n-f)}{n-2f} (2\bar{\nu} + \bar{\zeta})c_{max} + \frac{4c_{max}}{n-2f}\underbrace{ \sum_{i \in \mathcal{R}} \left ( C_{\mathcal{M},i}^2 + C_\mathcal{B}^2 \right )}_{T_2}
        \label{Flkappa},
\end{align}
where the second inequality holds as $ c_{max} := \max\{c_i, i\in[N]\}$ and the last inequality follows Lemma~\ref{local_divergence}.

\noindent Assume the benign model update is bounded as $\|\Delta_i\|^2 \leq C^2, \forall i\in \mathcal{B}$, which can be achieved by using gradient clipping in practice. Assume the malicious model update is bounded as $\|\Delta_i\|^2 \leq C_{\lambda_m}^2, \forall i\in [N]\backslash \mathcal{B}$, which depends on the specific attack method and our magnitude-based filtering that is controlled by $\lambda_m$ in Algorithm 1. We have
\begin{align}
    T_2 = \sum_{i \in \mathcal{R}} \left ( C_{\mathcal{M},i}^2 + C_\mathcal{B}^2 \right ) \leq |\mathcal{R}|\left(C_{\lambda_m}^2+C^2\right) \leq f \left(C_{\lambda_m}^2+C^2\right),
\end{align}
as $|\mathcal{R}|\leq |[N]\backslash  \mathcal{B}|\leq f$.
Therefore, 
\begin{align}
    T_1 & \leq c_{max}\left(\frac{2(n-f)}{n-2f} (2\bar{\nu} + \bar{\zeta}) + \frac{4f}{n-2f} \left(C_{\lambda_m}^2+C^2\right)\right)\nonumber\\
    & \leq c_k \left(\frac{2(n-f)}{n-2f} (2\bar{\nu} + \bar{\zeta}) + \frac{4f}{n-2f} \left(C_{\lambda_m}^2+C^2\right)\right)\nonumber\\
    & \leq c_k\left(1+ \frac{f}{n-2f}\right)(4\bar{\nu} + 2\bar{\zeta}+4C_{\lambda_m}^2+4C^2)\label{final_T1},
\end{align}
if the sparsification applied to the local model update satisfies Assumption 4 so that $ c_{max}\leq c_k$.

\noindent Summarizing to \eqref{T_1_start}, we have
\begin{align}
    \mathbb{E}\|F(\Delta_1, \dots, \Delta_n) - \bar{\Delta}_\mathcal{B}\|^2 & \leq {T_1}+ C_{\mathcal{B}}^2b_{\mathcal{B}}\nonumber\\
    & \leq T_1 + b_k C^2\nonumber\\
    & \leq c_k\left(1+ \frac{f}{n-2f}\right)(4\bar{\nu} + 2\bar{\zeta}+4C_{\lambda_m}^2+4C^2) + b_k C^2
\end{align}

\textbf{Discussion on the selection of $k$:}
When no sparsification is applied, i.e., when $k=d$, we have $c_k=1$ and $b_k=0$. In this case, the robustness upper bound is
\begin{equation*}
    \kappa_1 = \left(1+ \frac{f}{n-2f}\right)\left(4\bar{\nu} + 2\bar{\zeta}+4C_{\lambda_m}^2+4C^2\right) = O\left(1+ \frac{f}{n-2f}\right).
\end{equation*}
When $k=0$, we have $c_k=0$ and $b_k=1$, then
\begin{equation*}
    \kappa = C^2,
\end{equation*}
which indicates the greatest sparsification error affecting robustness. When $0<k<d$, the robustness upper bound is
\begin{equation*}
    \kappa_2 = (1+\epsilon)c_k\left(1+ \frac{f}{n-2f}\right)\left(4\bar{\nu} + 2\bar{\zeta}+4C_{\lambda_m}^2+4C^2\right) = O\left(c_k\left(1+ \frac{f}{n-2f}\right)\right)
\end{equation*}
if the sparsification parameter $k$ is selected to satisfy that
\begin{equation*}
    \text{Condition 1}: c_k\leq \frac{1}{1+\epsilon}
\end{equation*}
and 
\begin{equation*}
    \text{Condition 2}: \frac{b_k}{c_k} \leq \epsilon \left(\frac{4\bar{\nu} + 2\bar{\zeta}+4C_{\lambda_m}^2}{C^2} + 4\right)
\end{equation*}
with a positive constant $\epsilon$. As $(1+\epsilon)c_k\leq 1$, we have
\begin{equation*}
    \kappa_2\leq \kappa_1,
\end{equation*}
which demonstrates the effectiveness of sparsification for improving robustness.
This finally concludes the proof. 
\end{proof}

\clearpage

\subsubsection{Proof of Theorem 1} 
\label{proofofthe1}

\begin{proof}
    Given the update rule $\theta^{t+1} = \theta^{t} - \bar{\Delta}^t = \theta^t - \eta \tilde{\Delta}^t$ where $ \tilde{\Delta}^t_i := \sum ^{\tau-1}_{r=0}g^{t,r}_i= \tau d_i^t$, for ease of expression, we let $ \tilde{\Delta}_{\mathcal{B}^t}:= \frac{1}{|\mathcal{B}^t|}\sum_{i\in\mathcal{B}^t} \tilde{\Delta}^t_i$ and $h_i^t := \mathbb{E}[d_i^t]= \mathbb{E}\left[({1}/{\tau})\sum_{r=0}^{r=\tau-1}g_i^{t,r}\right] = ({1}/{\tau})\sum_{r=0}^{r=\tau-1}\nabla\mathcal{L}_{i}(\theta_i^{t,r})$. With Assumption 1, we have the following for all $t\in[0,T-1]$:
    \begin{align}
     \nonumber \mathcal{L}_\mathcal{B}(\theta^{t+1}) - \mathcal{L}_\mathcal{B}(\theta^{t}) &\leq \mathbb{E}\left \langle \nabla\mathcal{L}_\mathcal{B}(\theta^{t}), \theta^{t+1} - \theta^{t} \right \rangle + \frac{\mu}{2} \mathbb{E}\left \| \theta^{t+1} - \theta^{t} \right \|^2 \\ \nonumber 
    &= -\eta\mathbb{E} \left \langle \nabla\mathcal{L}_\mathcal{B}(\theta^{t}), \tilde{\Delta}^t \right \rangle+ \frac{\mu \eta^2}{2}\mathbb{E}\left \|  \tilde{\Delta}^t \right \|^2 \\ \nonumber 
    &= -\eta\mathbb{E} \left \langle \nabla\mathcal{L}_\mathcal{B}(\theta^{t}), \tilde{\Delta}^t + \tilde{\Delta}_{\mathcal{B}^t} - \tilde{\Delta}_{\mathcal{B}^t}\right \rangle+ \frac{\mu \eta^2}{2}\mathbb{E}\left \|  \tilde{\Delta}^t \right \|^2 \\ \nonumber 
    &= -\eta\mathbb{E} \left \langle \nabla\mathcal{L}_\mathcal{B}(\theta^{t}), \tilde{\Delta}_{\mathcal{B}^t} \right \rangle 
    -\eta\mathbb{E} \left \langle \nabla\mathcal{L}_\mathcal{B}(\theta^{t}), \tilde{\Delta}^t  - \tilde{\Delta}_{\mathcal{B}^t}\right \rangle
    + \frac{\mu \eta^2}{2}\mathbb{E}\left \|  \tilde{\Delta}^t \right \|^2 \\ 
    & = \underbrace{-\eta\mathbb{E} \left \langle \nabla\mathcal{L}_\mathcal{B}(\theta^{t}), \frac{1}{|\mathcal{B}^t|} \sum_{i \in \mathcal{B}^t} \tilde{\Delta}^t_i  \right \rangle}_{T_1} 
    + \underbrace{\eta\mathbb{E} \left \langle \nabla\mathcal{L}_\mathcal{B}(\theta^{t}), \tilde{\Delta}_{\mathcal{B}^t} -\tilde{\Delta}^t\right \rangle}_{T_2}
    + \underbrace{\frac{\mu \eta^2}{2}\mathbb{E}\left \|  \tilde{\Delta}^t \right \|^2}_{T_3}. \label{caog}
\end{align}
Now we treat $T_1$, $ T_2$, and $T_3$ respectively. We decompose $T_1$ by 
\begin{align}
    T_1 & =  \nonumber -\eta\mathbb{E} \left \langle \nabla\mathcal{L}_\mathcal{B}(\theta^{t}), \frac{1}{|\mathcal{B}^t|} \sum_{i \in \mathcal{B}^t} \tilde{\Delta}^t_i  \right \rangle
     = -\eta\tau\mathbb{E} \left \langle \nabla\mathcal{L}_\mathcal{B}(\theta^{t}), \frac{1}{|\mathcal{B}^t|} \sum_{i \in \mathcal{B}^t} d_i^t \right \rangle
     = -\eta\tau\mathbb{E} \left \langle \nabla\mathcal{L}_\mathcal{B}(\theta^{t}), \frac{1}{|\mathcal{B}|} \sum_{i \in \mathcal{B}} h_i^t \right \rangle
    \\
    & =\frac{\eta\tau }{2} \mathbb{E}\left \| \frac{1}{|\mathcal{B}|} \sum_{i \in \mathcal{B}} h^t_i - \nabla\mathcal{L}_\mathcal{B}(\theta^{t}) \right \|^2 -\frac{\eta \tau}{2} \mathbb{E} \left \| \nabla\mathcal{L}_\mathcal{B}(\theta^{t}) \right \|^2 - \frac{\eta \tau}{2} \mathbb{E} \left \| \frac{1}{|\mathcal{B}|} \sum_{i \in \mathcal{B}} h^t_i \right \|^2,\label{T1inter}
\end{align}
where we use the fact that $-2\left \langle a, b \right \rangle = \left \| a - b \right \|^2 - \left \| a \right\|^2 - \left \| b \right\|^2$.

We decompose $T_2$ as  
\begin{align}
    T_2 & = \eta\mathbb{E} \left \langle \nabla\mathcal{L}_\mathcal{B}(\theta^{t}), \tilde{\Delta}_{\mathcal{B}^t}-\tilde{\Delta}^t\right \rangle \leq \frac{\eta\alpha}{2}\mathbb{E} \left \| \nabla\mathcal{L}_\mathcal{B}(\theta^{t})\right \|^2 + \frac{\eta}{2\alpha}\mathbb{E} \left\|\tilde{\Delta}^t - \tilde{\Delta}_{\mathcal{B}^t}\right \|^2 ,\label{T2inter}
\end{align}
where the first inequality follows Lemma~\ref{lemma_for_product} with a $\alpha > 0$. 

We decompose $T_3$ as 
\begin{align}
     \nonumber T_3 = \frac{\mu \eta^2}{2}\mathbb{E}\left \|  \tilde{\Delta}^t \right \|^2 
    &  = \frac{\mu \eta^2}{2}\mathbb{E}\left \|  \tilde{\Delta}^t + \tilde{\Delta}_{\mathcal{B}^t} - \tilde{\Delta}_{\mathcal{B}^t}\right \|^2\\ \nonumber 
    & \leq {\mu \eta^2}\mathbb{E}\left \|   \tilde{\Delta}_{\mathcal{B}^t}\right \|^2 + \mu \eta^2\mathbb{E}\left \|  \tilde{\Delta}^t - \tilde{\Delta}_{\mathcal{B}^t}\right \|^2\\ \nonumber 
    & = {\mu \eta^2}\mathbb{E}\left \|   \frac{1}{|\mathcal{B}^t|} \sum_{i \in \mathcal{B}^t} \tilde{\Delta}^t_i\right \|^2 + \mu \eta^2\mathbb{E}\left \|  \tilde{\Delta}^t - \tilde{\Delta}_{\mathcal{B}^t}\right \|^2 \\
    & \leq \frac{\mu \eta^2}{|\mathcal{B}|} \sum_{i \in \mathcal{B}} \mathbb{E}\left \|  \tilde{\Delta}^t_i\right \|^2 + \mu \eta^2\mathbb{E}\left \|  \tilde{\Delta}^t - \tilde{\Delta}_{\mathcal{B}^t}\right \|^2 ,\label{T3inter}
\end{align}
where the first inequality follows Lemma~\ref{lemma_for_break_norm} with $\delta = 1$ and the second inequality follows Lemma~\ref{lemma_for_extract_sum}. 

Combining~\ref{T1inter},~\ref{T2inter},~\ref{T3inter} and,~\ref{caog}, we get
\begin{align}
     \nonumber \mathcal{L}_\mathcal{B}(\theta^{t+1}) - \mathcal{L}_\mathcal{B}(\theta^{t})  
    & \leq \frac{\eta\tau }{2} \mathbb{E}\left \| \frac{1}{|\mathcal{B}|} \sum_{i \in \mathcal{B}} h^t_i - \nabla\mathcal{L}_\mathcal{B}(\theta^{t}) \right \|^2 -\frac{\eta \tau}{2} \mathbb{E} \left \| \nabla\mathcal{L}_\mathcal{B}(\theta^{t}) \right \|^2 - \frac{\eta \tau}{2} \mathbb{E} \left \| \frac{1}{|\mathcal{B}|} \sum_{i \in \mathcal{B}} h^t_i \right \|^2  \\ \nonumber 
     & \quad \ + \frac{\eta\alpha}{2}\mathbb{E} \left \| \nabla\mathcal{L}_\mathcal{B}(\theta^{t})\right \|^2 + \frac{\eta}{2\alpha}\mathbb{E} \left\|\tilde{\Delta}^t - \tilde{\Delta}_{\mathcal{B}^t}\right \|^2+ \frac{\mu \eta^2}{|\mathcal{B}|} \sum_{i \in \mathcal{B}} \mathbb{E}\left \|  \tilde{\Delta}^t_i\right \|^2 + \mu \eta^2\mathbb{E}\left \|  \tilde{\Delta}^t - \tilde{\Delta}_{\mathcal{B}^t}\right \|^2\\ \nonumber 
    & = -\left(\frac{\eta \tau}{2} - \frac{\eta\alpha}{2}\right) \mathbb{E} \left \| \nabla\mathcal{L}_\mathcal{B}(\theta^{t}) \right \|^2 + \frac{\eta\tau }{2} \mathbb{E}\left \| \frac{1}{|\mathcal{B}|} \sum_{i \in \mathcal{B}} h^t_i  - \nabla\mathcal{L}_\mathcal{B}(\theta^{t}) \right \|^2 \\ \nonumber 
    & \quad \ + \left(\mu \eta^2 + \frac{\eta}{2\alpha}\right)\mathbb{E}\left \|  \tilde{\Delta}^t - \tilde{\Delta}_{\mathcal{B}^t}\right \|^2 - \frac{\eta \tau}{2} \mathbb{E} \left \| \frac{1}{|\mathcal{B}|} \sum_{i \in \mathcal{B}} h^t_i \right \|^2 + \frac{\mu \eta^2}{|\mathcal{B}|} \sum_{i \in \mathcal{B}} \mathbb{E}\left \|  \tilde{\Delta}^t_i\right \|^2 \\ \nonumber 
    & = -\left(\frac{\eta \tau}{2} - \frac{\eta\alpha}{2}\right) \mathbb{E} \left \| \nabla\mathcal{L}_\mathcal{B}(\theta^{t}) \right \|^2 + \frac{\eta\tau }{2} \mathbb{E}\left \| \frac{1}{|\mathcal{B}|} \sum_{i \in \mathcal{B}} h^t_i  - \nabla\mathcal{L}_\mathcal{B}(\theta^{t}) \right \|^2 \\
    & \quad \ + \underbrace{\left(\mu \eta^2 + \frac{\eta}{2\alpha}\right)\mathbb{E}\left \|  \tilde{\Delta}^t - \tilde{\Delta}_{\mathcal{B}^t}\right \|^2}_{T_4} - \frac{\eta \tau}{2} \mathbb{E} \left \| \frac{1}{|\mathcal{B}|} \sum_{i \in \mathcal{B}} h^t_i \right \|^2 + \frac{\mu \eta^2}{|\mathcal{B}|} \sum_{i \in \mathcal{B}} \mathbb{E}\left \|  \tilde{\Delta}^t_i\right \|^2. \label{og_inter_1}
\end{align}

$T_4$ can be decomposed as 
\begin{align}
      T_4 &= \left(\mu \eta^2 + \frac{\eta}{2\alpha}\right)\mathbb{E}\left \|  \tilde{\Delta}^t - \tilde{\Delta}_{\mathcal{B}^t}\right \|^2 \leq \kappa\left(\mu \eta^2 + \frac{\eta}{2\alpha}\right) \label{T4}
\end{align}
where the first inequality holds as LASA is $\kappa$-robust aggregation rule with $\kappa$.

Plugging~\ref{T4} back to~\ref{og_inter_1}, we have
\begin{align}
      \nonumber    \mathcal{L}_\mathcal{B}(\theta^{t+1}) - \mathcal{L}_\mathcal{B}(\theta^{t}) 
    & \leq -\left(\frac{\eta \tau}{2} - \frac{\eta\alpha}{2}\right) \mathbb{E} \left \| \nabla\mathcal{L}_\mathcal{B}(\theta^{t}) \right \|^2 + \frac{\eta\tau }{2} \mathbb{E}\left \| \frac{1}{|\mathcal{B}|} \sum_{i \in \mathcal{B}} h^t_i  - \nabla\mathcal{L}_\mathcal{B}(\theta^{t}) \right \|^2 \\ \nonumber 
    & \quad \ + \kappa\left(\mu \eta^2 + \frac{\eta}{2\alpha}\right) - \frac{\eta \tau}{2} \mathbb{E} \left \| \frac{1}{|\mathcal{B}|} \sum_{i \in \mathcal{B}} h^t_i \right \|^2 + \frac{\mu \eta^2}{|\mathcal{B}|} \sum_{i \in \mathcal{B}} \mathbb{E}\left \|  \tilde{\Delta}^t_i\right \|^2
    \\ \nonumber 
    &= -\left(\frac{\eta \tau}{2} - \frac{\eta\alpha}{2}\right) \mathbb{E} \left \| \nabla\mathcal{L}_\mathcal{B}(\theta^{t}) \right \|^2 + \frac{\eta\tau }{2} \mathbb{E}\left \| \frac{1}{|\mathcal{B}|} \sum_{i \in \mathcal{B}} h^t_i  - \nabla\mathcal{L}_\mathcal{B}(\theta^{t}) \right \|^2 \\
    & \quad \ +  \kappa\left(\mu \eta^2 + \frac{\eta}{2\alpha}\right) +  \mu \eta^2  \underbrace{\frac{1}{|\mathcal{B}|} \sum_{i \in \mathcal{B}}\mathbb{E}\left \|  \tilde{\Delta}_i^t\right \|^2}_{T_5} - \frac{\eta \tau}{2} \mathbb{E} \left \| \frac{1}{|\mathcal{B}|} \sum_{i \in \mathcal{B}} h^t_i \right \|^2  \label{og_inter_2}
\end{align}

$T_5$ can be charcterized as 
\begin{align}
     \nonumber T_5 = \frac{1}{|\mathcal{B}|} \sum_{i \in \mathcal{B}}\mathbb{E}\left \|  \tilde{\Delta}_i^t\right \|^2 &= \frac{\tau^2}{|\mathcal{B}|} \sum_{i \in \mathcal{B}}\mathbb{E}\left \| d_i^t\right \|^2 = \frac{\tau^2}{|\mathcal{B}|} \sum_{i \in \mathcal{B}}\left(\mathbb{E} \left \| d_i^t - h_i^t\right \|^2 +\mathbb{E}\left \| h_i^t\right \|^2 \right) \\ \nonumber 
    &=  \frac{\tau^2}{|\mathcal{B}|} \sum_{i \in \mathcal{B}}\left(\mathbb{E}\left\| \frac{1}{\tau}\sum_{s=0}^{\tau-1}\left(g_i^{t,s} - \nabla \mathcal{L}_{ i}(\theta_i^{t,s})\right)\right\|^2 +\mathbb{E}\left \| h_i^t\right \|^2 \right) \\ \nonumber 
    &\leq  \frac{\tau^2}{|\mathcal{B}|} \sum_{i \in \mathcal{B}}\left(\frac{1}{\tau}\sum_{s=0}^{\tau-1}\mathbb{E}\left\| g_i^{t,s} - \nabla \mathcal{L}_{i}(\theta_i^{t,s})\right\|^2 +\mathbb{E}\left \| h_i^t\right \|^2 \right) \\ \nonumber 
    &\leq  \frac{\tau^2}{|\mathcal{B}|} \sum_{i \in \mathcal{B}}\left(\frac{1}{\tau}\sum_{s=0}^{\tau-1} \nu_i^2 +\mathbb{E}\left \| h_i^t\right \|^2 \right) \\
    &=  \frac{\tau^2}{|\mathcal{B}|} \sum_{i \in \mathcal{B}}\left(\nu_i^2 +\mathbb{E}\left \| h_i^t\right \|^2 \right), \label{T5}
\end{align}
where the first inequality follows Lemma~\ref{lemma_for_extract_sum} and the second inequality follows Assumption 2.

Plugging~\ref{T5} back to~\ref{og_inter_2}, we have

\begin{align}
     \nonumber   \mathcal{L}_\mathcal{B}(\theta^{t+1}) - \mathcal{L}_\mathcal{B}(\theta^{t}) 
    &\leq -\left(\frac{\eta \tau}{2} - \frac{\eta\alpha}{2}\right) \mathbb{E} \left \| \nabla\mathcal{L}_\mathcal{B}(\theta^{t}) \right \|^2 + \frac{\eta\tau }{2} \mathbb{E}\left \| \frac{1}{|\mathcal{B}|} \sum_{i \in \mathcal{B}} h^t_i  - \nabla\mathcal{L}_\mathcal{B}(\theta^{t}) \right \|^2 - \frac{\eta \tau}{2} \mathbb{E} \left \| \frac{1}{|\mathcal{B}|} \sum_{i \in \mathcal{B}} h^t_i \right \|^2\\ \nonumber 
    & \quad \ +  \frac{\mu \eta^2 \tau^2}{|\mathcal{B}|} \sum_{i \in \mathcal{B}}\left(\nu_i^2 +\mathbb{E}\left \| h_i^t\right \|^2 \right) + \kappa\left(\mu \eta^2 + \frac{\eta}{2\alpha}\right)  \\ \nonumber 
    &= -\left(\frac{\eta \tau}{2} - \frac{\eta\alpha}{2}\right) \mathbb{E} \left \| \nabla\mathcal{L}_\mathcal{B}(\theta^{t}) \right \|^2 + \frac{\eta^2\tau }{2} \mathbb{E}\left \| \frac{1}{|\mathcal{B}|} \sum_{i \in \mathcal{B}} h^t_i  - \nabla\mathcal{L}_\mathcal{B}(\theta^{t}) \right \|^2  +  \frac{\mu \eta^2 \tau^2}{|\mathcal{B}|} \sum_{i \in \mathcal{B}}\mathbb{E}\left \| h_i^t\right \|^2   \\
    & \quad  \ +  \mu \eta \tau^2\bar{\nu}^2 + \kappa\left(\mu \eta^2 + \frac{\eta}{2\alpha}\right) \\ \nonumber 
    &\leq -\left(\frac{\eta \tau}{2} - \frac{\eta\alpha}{2}\right) \mathbb{E} \left \| \nabla\mathcal{L}_\mathcal{B}(\theta^{t}) \right \|^2 + \frac{\eta\tau }{2}\frac{1}{|\mathcal{B}|} \sum_{i \in \mathcal{B}}\frac{1}{\tau}\sum_{r=0}^{\tau-1} \mathbb{E}\left \|  \nabla\mathcal{L}_{ i}(\theta_i^{t,r})  - \nabla\mathcal{L}_{i}(\theta^{t}) \right \|^2 \\ \nonumber 
    & \quad \ +  \frac{\mu \eta^2 \tau^2}{|\mathcal{B}|} \sum_{i \in \mathcal{B}} \left ( 2\mathbb{E}\left \| h_i^t - \nabla \mathcal{L}_{i}(\theta^t)\right \|^2 + \frac{2}{|\mathcal{B}|} \sum_{i \in \mathcal{B}}\mathbb{E} \left \| \nabla\mathcal{L}_{i} (\theta^{t}) \right \|^2  \right) +  \mu \eta^2 \tau^2\bar{\nu}^2 + \kappa\left(\mu \eta^2 + \frac{\eta}{2\alpha}\right) \\ \nonumber 
    &\leq -\left(\frac{\eta \tau}{2} - \frac{\eta\alpha}{2}\right) \mathbb{E} \left \| \nabla\mathcal{L}_\mathcal{B}(\theta^{t}) \right \|^2 + \frac{\eta\tau }{2}\frac{1}{|\mathcal{B}|} \sum_{i \in \mathcal{B}}\frac{\mu^2}{\tau}\sum_{r=0}^{\tau-1} \mathbb{E}\left \|  \theta_i^{t,r} - \theta^{t} \right \|^2 + \frac{\mu \eta^2\tau^2}{|\mathcal{B}|} \sum_{i \in \mathcal{B}} \frac{\mu^2}{\tau} \sum_{r=0}^{\tau - 1}  2\mathbb{E}\left \| \theta_i^{t,r} - \theta^{t} \right \|^2\\ \nonumber 
    & \quad \  +  
    \frac{4\mu \eta^2 \tau^2}{|\mathcal{B}|} \sum_{i \in \mathcal{B}} \frac{1}{|\mathcal{B}|}\sum_{i \in \mathcal{B}} \mathbb{E} (\bar{\zeta} +  \left \| \nabla\mathcal{L}_{\mathcal{B}} (\theta^{t}) \right \|^2 )
    + \mu \eta^2 \tau^2\bar{\nu}^2 + \kappa\left(\mu \eta^2 + \frac{\eta}{2\alpha}\right) \\ \nonumber 
    &= \left[ -\left(\frac{\eta \tau}{2} - \frac{\eta\alpha}{2}\right) + 4 \mu \eta^2 \tau^2 \right] \mathbb{E} \left \| \nabla\mathcal{L}_\mathcal{B}(\theta^{t}) \right \|^2 + \left (\frac{\eta \mu^2 }{2} + 2\eta^2 \tau \mu^3 \right)\sum_{r=0}^{\tau-1}\underbrace{\frac{1}{|\mathcal{B}|} \sum_{i \in \mathcal{B}} \mathbb{E}\left \|  \theta_i^{t,r} - \theta^{t} \right \|^2}_{T_6} \\
    & \quad \  + 4 \mu \eta^2 \tau^2 \bar{\zeta}+ \mu \eta^2 \tau^2\bar{\nu}^2 + \kappa\left(\mu \eta^2 + \frac{\eta}{2\alpha}\right)
    \label{og_inter_3}
\end{align}
where the second inequality follows Lemma~\ref{lemma_for_break_norm} and the third inequality follow Assumption 1.

\noindent Now we treat $T_6$ as
\begin{align}
    T_6 &= \frac{1}{|\mathcal{B}|} \sum_{i \in \mathcal{B}} \mathbb{E}\left \|  \theta_i^{t,r} - \theta^{t} \right \|^2 = \frac{1}{|\mathcal{B}|} \sum_{i \in \mathcal{B}} \mathbb{E}\left \|  \theta_i^{t,r-1} - \theta^{t} -\eta g^{t, s-1}_i\right \|^2 \nonumber \\
    &= \frac{1}{|\mathcal{B}|} \sum_{i \in \mathcal{B}} \mathbb{E}\left \|  \theta_i^{t,r-1} - \theta^{t} -\eta g^{t, s-1}_i + \eta \nabla \mathcal{L}_i (\theta^{t, s-1}) - \eta \nabla \mathcal{L}_i (\theta^{t, s-1}) + \eta \nabla \mathcal{L}_i (\theta^{t}) - \eta \nabla \mathcal{L}_i (\theta^{t}) \right \|^2 \nonumber \\
    &= \frac{1}{|\mathcal{B}|} \sum_{i \in \mathcal{B}} \mathbb{E}\left \|\theta_i^{t,r-1} - \theta^{t} -\eta \nabla \mathcal{L}_i (\theta^{t, s-1}) + \eta \nabla \mathcal{L}_i (\theta^{t}) - \eta \nabla \mathcal{L}_i (\theta^{t}) \right \|^2 + \frac{\eta^2}{|\mathcal{B}|} \sum_{i \in \mathcal{B}} \mathbb{E} \left \| g^{t, s-1}_i -   \nabla \mathcal{L}_i (\theta^{t, s-1}) \right \|^2 \nonumber \\
    &\leq \left(1 + \frac{1}{2\tau - 1} \right) \frac{1}{|\mathcal{B}|} \sum_{i \in \mathcal{B}} \mathbb{E}\left \|\theta_i^{t,r-1} - \theta^{t} \right \|^2 + \frac{ 2 \tau \eta^2}{|\mathcal{B}|} \sum_{i \in \mathcal{B}} \left \|  \nabla \mathcal{L}_i (\theta^{t, s-1}) +  \nabla \mathcal{L}_i (\theta^{t}) -  \nabla \mathcal{L}_i (\theta^{t}) \right \|^2 + \eta^2 \bar{\nu} \nonumber \\
    &\leq \left(1 + \frac{1}{2\tau - 1} \right) \frac{1}{|\mathcal{B}|} \sum_{i \in \mathcal{B}} \mathbb{E}\left \|\theta_i^{t,r-1} - \theta^{t} \right \|^2 + \frac{ 4 \tau \eta^2}{|\mathcal{B}|} \sum_{i \in \mathcal{B}} \left \|  \nabla \mathcal{L}_i (\theta^{t, s-1}) -  \nabla \mathcal{L}_i (\theta^{t}) \right \|^2 + \frac{ 4 \tau \eta^2}{|\mathcal{B}|} \sum_{i \in \mathcal{B}} \left\| \nabla \mathcal{L}_i (\theta^{t}) \right \|^2 + \eta^2 \bar{\nu} \nonumber \\
    &\leq \left(1 + \frac{1}{2\tau - 1} \right) \frac{1}{|\mathcal{B}|} \sum_{i \in \mathcal{B}} \mathbb{E}\left \|\theta_i^{t,r-1} - \theta^{t} \right \|^2 + \frac{ 4 \tau \mu^2 \eta^2}{|\mathcal{B}|} \sum_{i \in \mathcal{B}} \left \|  \theta_i^{t,r-1} - \theta^{t}\right \|^2 + \frac{ 4\tau \eta^2}{|\mathcal{B}|} \sum_{i \in \mathcal{B}} \left\| \nabla \mathcal{L}_i (\theta^{t}) - \nabla \mathcal{L}_\mathcal{B} (\theta^{t}) + \nabla \mathcal{L}_\mathcal{B} (\theta^{t}) \right \|^2 + \eta^2 \bar{\nu} \nonumber \\
    &\leq \left(1 + \frac{1}{2\tau - 1} + 4 \tau \mu^2 \eta^2 \right) \frac{1}{|\mathcal{B}|} \sum_{i \in \mathcal{B}} \mathbb{E}\left \|\theta_i^{t,r-1} - \theta^{t} \right \|^2  + \frac{ 8 \tau \eta^2}{|\mathcal{B}|} \sum_{i \in \mathcal{B}} \left\| \nabla \mathcal{L}_\mathcal{B} (\theta^{t}) \right \|^2 + 8 \tau \bar{\zeta}\eta^2 + \eta^2 \bar{\nu} \nonumber \\
    &\leq \left(1 + \frac{1}{\tau - 1}  \right) \frac{1}{|\mathcal{B}|} \sum_{i \in \mathcal{B}} \mathbb{E}\left \|\theta_i^{t,r-1} - \theta^{t} \right \|^2 + \frac{ 8 \tau \eta^2}{|\mathcal{B}|} \sum_{i \in \mathcal{B}} \left\| \nabla \mathcal{L}_\mathcal{B} (\theta^{t}) \right \|^2 + 8 \tau \bar{\zeta}\eta^2 + \eta^2 \bar{\nu}, 
\end{align}
where the first and second inequality follows Lemma~\ref{lemma_for_break_norm} with $\delta = 2\tau$ and $\delta = 1$, respectively. The third inequality follows Assumption 1, and the last inequality holds if $\eta \leq 1/3 \tau \mu$.
Consequently, we have
\begin{align}
T_6 &= \frac{1}{|\mathcal{B}|} \sum_{i \in \mathcal{B}} \mathbb{E}\left \|  \theta_i^{t,r} - \theta^{t} \right \|^2 
\leq \sum^{s-1}_{h=0} \left(1 + \frac{1}{\tau - 1}  \right)^h \left[  8 \tau \eta^2 \left\| \nabla \mathcal{L}_\mathcal{B} (\theta^{t}) \right \|^2 + 8 \tau \bar{\zeta}\eta^2 + \eta^2 \bar{\nu} \right] \nonumber \\
& \leq (\tau - 1)\left [\left(1 + \frac{1}{\tau - 1}  \right)^\tau -1 \right ] \times \left [ 8 \tau \eta^2 \left\| \nabla \mathcal{L}_\mathcal{B} (\theta^{t}) \right \|^2 + 8 \tau \bar{\zeta}\eta^2 + \eta^2 \bar{\nu} \right] \nonumber \\
& \leq 32 \tau^2 \eta^2 \left\| \nabla \mathcal{L}_\mathcal{B} (\theta^{t}) \right \|^2 + 32 \tau^2 \bar{\zeta}\eta^2 + 4 \tau \eta^2 \bar{\nu} , \label{T_6}
\end{align}
where the last inequality results from the fact that $\left( 1 + \frac{1}{\tau - 1} \right)^t \leq 5$ when $\tau > 1$.

Plugging~\ref{T_6} back to~\ref{og_inter_3}, we have
\begin{align}
     \nonumber \mathcal{L}_\mathcal{B}(\theta^{t+1}) - \mathcal{L}_\mathcal{B}(\theta^{t}) 
    &\leq \left[ -\left(\frac{\eta \tau}{2} - \frac{\eta\alpha}{2}\right) + 4 \mu \eta^2 \tau^2 \right] \mathbb{E} \left \| \nabla\mathcal{L}_\mathcal{B}(\theta^{t}) \right \|^2  + 4 \mu \eta^2 \tau^2 \bar{\zeta}+ \mu \eta^2 \tau^2\bar{\nu}^2 + \kappa\left(\mu \eta^2 + \frac{\eta}{2\alpha}\right) \\ \nonumber 
    & \quad \ + \left (\frac{\eta \mu^2 }{2} + 2\eta^2 \tau \mu^3 \right)\sum_{r=0}^{\tau-1} \left [ 32 \tau^2 \eta^2 \left\| \nabla \mathcal{L}_\mathcal{B} (\theta^{t}) \right \|^2 + 32 \tau^2 \bar{\zeta}\eta^2 + 4 \tau \eta^2 \bar{\nu}\right] \\ \nonumber 
    &= \left[ -\left(\frac{\eta \tau}{2} - \frac{\eta\alpha}{2}\right) + 4 \mu \eta^2 \tau^2 \right] \mathbb{E} \left \| \nabla\mathcal{L}_\mathcal{B}(\theta^{t}) \right \|^2  + 4 \mu \eta^2 \tau^2 \bar{\zeta}+ \mu \eta^2 \tau^2\bar{\nu}^2 + \kappa\left(\mu \eta^2 + \frac{\eta}{2\alpha}\right) \\ \nonumber 
    & \quad \ + \left (\frac{\eta \mu^2 }{2} + 2\eta^2 \tau \mu^3 \right)\left [ 32 \tau^3 \eta^2 \left\| \nabla \mathcal{L}_\mathcal{B} (\theta^{t}) \right \|^2 + 32 \tau^3 \bar{\zeta}\eta^2 + 4 \tau^2 \eta^2 \bar{\nu} \right] \\ \nonumber 
    &= \left [ \left[ -\left(\frac{\eta \tau}{2} - \frac{\eta\alpha}{2}\right) + 4 \mu \eta^2 \tau^2 \right] + \left (\frac{\eta \mu^2 }{2} + 2\eta^2 \tau \mu^3 \right) \left(32 \eta^2 \tau^3  \right) \right] \mathbb{E} \left \| \nabla\mathcal{L}_\mathcal{B}(\theta^{t}) \right \|^2 \\ \nonumber 
    & \quad \ + \left (\frac{\eta \mu^2 }{2} + 2\eta^2 \tau \mu^3 \right) \left( 32 \tau^3 \bar{\zeta}\eta^2 + 4 \tau^2 \eta^2 \bar{\nu}\right)+ 4 \mu \eta^2 \tau^2 \bar{\zeta}+ \mu \eta^2 \tau^2\bar{\nu} + \kappa\left(\mu \eta^2 + \frac{\eta}{2\alpha}\right) \\ 
    &\leq -\eta \mathbb{E} \left \| \nabla\mathcal{L}_\mathcal{B}(\theta^{t}) \right \|^2  + \left (\frac{\eta \mu^2 }{2} + 2\eta^2 \tau \mu^3 \right) \left( 32 \tau^3 \bar{\zeta}\eta^2 + 4 \tau^2 \eta^2 \bar{\nu}\right)+ 4 \mu \eta^2 \tau^2 \bar{\zeta}+ \mu \eta^2 \tau^2\bar{\nu} + \kappa\left(\mu \eta^2 + \frac{\eta}{2\alpha}\right) \nonumber  \\ 
    &\leq -\eta \mathbb{E} \left \| \nabla\mathcal{L}_\mathcal{B}(\theta^{t}) \right \|^2 + \kappa\left(\mu \eta^2 + \frac{\eta}{4}\right) + 
         7 \eta \tau \bar{\zeta} + (1 + \tau) \eta  \bar{\nu}    \label{og_inter_4}
\end{align}
where the second inequality holds with $\alpha \geq 2$, and $ \eta  \leq 1 / 3\mu \tau$.

\noindent Times $1 / \eta$ to the both sides of~\ref{og_inter_4}, rearranging and summing it form $t=0$ to $t=T-1$ and dividing by $T$, one yields
\begin{align}
     \nonumber \frac{1}{T}\sum^{T-1}_{t=0}\left \| \nabla\mathcal{L}_\mathcal{B}(\theta^{t}) \right \|^2 \leq \frac{ \left(\mathcal{L}_\mathcal{B}(\theta^{0}) - \mathcal{L}_\mathcal{B}(\theta^{*}) \right)}{T\eta}  + \kappa\left(\mu \eta + 1 \right) + 
         7  \tau \bar{\zeta} + (1 + \tau)   \bar{\nu}.
\end{align}
Assume $\widetilde{\theta}$ is uniformly sampled from the sequence of outputs $\{\theta^0, \theta^1, \dots, \theta^T \}$ generated by FL with LASA as the $F$, then we have 
\[
\mathbb{E}\left \| \nabla\mathcal{L}_\mathcal{B}(\widetilde{\theta}) \right \|^2 = \frac{1}{T}\sum^{T-1}_{t=0}\left \| \nabla\mathcal{L}_\mathcal{B}(\theta^{t}) \right \|^2,
\]
which concludes the proof. 
\end{proof}

\end{document}